%% file: main.tex
\documentclass[twoside]{article}
\usepackage[utf8]{inputenc} % allow utf-8 input
\usepackage[T1]{fontenc}    % use 8-bit T1 fonts
\usepackage[pagebackref=true, colorlinks=true,
            linkcolor = blue, urlcolor  = blue,
            citecolor = blue, anchorcolor = blue]{hyperref}

\usepackage[margin=1.25in]{geometry}
\usepackage{fancyhdr}

\usepackage{url}            % simple URL typesetting
\usepackage{booktabs}       % professional-quality tables
\usepackage{amsfonts}       % blackboard math symbols
\usepackage{nicefrac}       % compact symbols for 1/2, etc.
\usepackage{microtype}      % microtypography
\usepackage{xcolor}         % colors

\usepackage{amsmath, amsthm, amsfonts, amssymb}
\usepackage{mathtools, thmtools}
\usepackage{graphicx}
\usepackage{xargs}
\usepackage{paralist}

\usepackage[capitalise, nameinlink]{cleveref}
\usepackage{bbm}

\usepackage{textcomp} % required for special glyphs
\usepackage[varg,bigdelims]{newpxmath}
\usepackage[scr=rsfso]{mathalfa}% \mathscr is fancier than \mathcal
\usepackage{bm} % load after all math to give access to bold math

\usepackage{algorithm}
\usepackage{algpseudocodex}

\usepackage{pifont}
\newcommand{\cmark}{\ding{51}}
\newcommand{\xmark}{\ding{55}}

\usepackage[symbol]{footmisc}
\usepackage{tablefootnote}
\usepackage{threeparttable}

\appto\TPTnoteSettings{\footnotesize}

\usepackage{caption}
\usepackage{enumitem}

\input{preamble.tex}

\title{Privacy Induces Robustness:
 Information-Computation Gaps and Sparse Mean Estimation}

\author{%
  Kristian Georgiev \thanks{Authors are in alphabetical order.} \\
  MIT EECS\\
  Cambridge, MA 02139\\
  \texttt{krisgrg@mit.edu} \\
  \and
  Samuel B. Hopkins\\
  MIT EECS\\
  Cambridge, MA 02139\\
  \texttt{samhop@mit.edu} \\
}

\newcommand\Author{Kristian Georgiev, Samuel B. Hopkins}
\makeatletter
\let\Title\@title
\makeatother

\begin{document}
\maketitle

\begin{abstract}
  \input{sections/abstract.tex}
\end{abstract}
\newpage
\tableofcontents
% \newpage

\input{sections/intro.tex}
\input{sections/prelims.tex}
\input{sections/metatheorem.tex}
\input{sections/hardness.tex}
\input{sections/sparse_mean.tex}
\input{sections/sparse_inf_theory.tex}
\input{sections/strongly-private-alg.tex}
\input{appendicies/proofs_sos.tex}
\input{appendicies/experiment_details.tex}
\input{appendicies/restate_cost_of_privacy.tex}
\input{appendicies/concentrated_DP.tex}
\input{sections/acknowledgements.tex}

\bibliography{biblio}
\appendix
\end{document}

%% file: preamble.tex
\newcommandx{\question}[2][1=]{\todo[linecolor=blue,
                                     backgroundcolor=blue!25,
                                     bordercolor=blue,#1]{#2}}

\newtheorem{theorem}{Theorem}[section]
\newtheorem*{theorem-non}{Theorem}
\newtheorem{lemma}[theorem]{Lemma}
\newtheorem*{lemma-non}{Lemma}
\newtheorem{corollary}[theorem]{Corollary}
\newtheorem{proposition}[theorem]{Proposition}
\newtheorem*{proposition-non}{Proposition}

\newtheorem{definition}[theorem]{Definition}
\newtheorem{conjecture}[theorem]{Conjecture}

\newcommand{\dtv}{d_{\mathrm{TV}}}
\newcommand{\Tr}{\mathrm{Tr}}
\newcommand{\acov}{\mathcal{C}_\alpha}
\newcommand{\opt}{\mathrm{OPT}}
\newcommand{\E}{\mathbb{E}}
\newcommand{\var}{\mathrm{Var}}
\renewcommand{\P}{\mathbb{P}}
\newcommand{\1}{\mathbbm{1}}
\newcommand{\maxv}{\max_{v\colon \|v\|=1}}

\newcommand{\xit}{x_{iT^*}}
\newcommand{\hmut}{\mu_{T^*}}
\newcommand{\poly}{\text{poly}}
\renewcommand{\epsilon}{\varepsilon}

\newcommand{\cC}{\mathcal{C}}

\newcommand{\cH}{\mathcal{H}}
\newcommand{\cN}{\mathcal{N}}
\newcommand{\cP}{\mathcal{P}}
\newcommand{\cQ}{\mathcal{Q}}

\newcommand{\R}{\mathbb{R}}
\newcommand{\cX}{\mathcal{X}}
\newcommand{\cO}{\mathcal{O}}
\newcommand{\N}{\mathbb{N}}
\newcommand{\e}{\varepsilon}
\newcommand{\pE}{\tilde{\mathbb{E}}}

\newcommand{\Paren}[1]{\left ( #1 \right )}
\newcommand{\iprod}[1]{\langle #1 \rangle}

\newcommand{\mper}{\, .}
\newcommand{\mcom}{\, ,}

\newcommand{\proves}{\vdash}

%% file: sections/abstract.tex
We establish a simple connection between robust and differentially-private
algorithms: private mechanisms \emph{which perform well with very high
probability} are automatically robust in the sense that they retain accuracy
even if a constant fraction of the samples they receive are adversarially
corrupted. Since optimal mechanisms typically %(though not universally) 
achieve
these high success probabilities, our results imply that optimal private
mechanisms for many basic statistics problems are robust.  

We investigate the consequences of this observation for both algorithms and
computational complexity across different statistical problems. 
Assuming the Brennan-Bresler secret-leakage planted clique conjecture, we
demonstrate a fundamental tradeoff between computational efficiency, privacy 
leakage, and success probability for sparse mean estimation. Private algorithms
which match this tradeoff are not yet known -- we achieve that (up to
polylogarithmic factors) in a polynomially-large range of parameters via the
Sum-of-Squares method.

To establish an information-computation gap for private sparse mean estimation, we
also design new (exponential-time) mechanisms using fewer samples than efficient
algorithms must use. 
Finally, we give evidence for privacy-induced information-computation gaps for
several other statistics and learning problems, including PAC learning parity
functions and estimation of the mean of a multivariate Gaussian. 

%% file: sections/intro.tex
\section{Introduction}\label{sec:intro}

Avoiding leakage of sensitive data and robustness to data corruption or model
misspecification are often key goals for designers of statistical
estimators. Both these properties admit mathematical formalizations, and a great
deal of recent work in (algorithmic) statistics has gone into designing and
analyzing algorithms that satisfy them.

\emph{Differential privacy} (DP) is the gold-standard formal definition of
privacy for algorithms processing sensitive data~\cite{dwork2006calibrating}.
DP requires that the distribution of outputs of an algorithm (or ``mechanism'')
$M$ is insensitive to exchanging a small number of individuals in its input
dataset. It offers such strong guarantees against privacy leakage that in
addition to its ongoing adoption in industry, the US Census Bureau employs DP to
satisfy its legal mandate to protect privacy~\cite{abowd2022geographic}.

The \emph{$\eta$-contamination model} is a stringent formalization of robustness against model misspecification. It generalizes the
classical model of i.i.d. samples: a dataset is $\eta$-corrupted for some $\eta
> 0$ if it is first drawn i.i.d, but then an $\eta$-fraction of samples have
been arbitrarily corrupted by a malicious adversary (who may look at the whole
dataset)~\cite{huber1965robust, tukey1975mathematics, huber2011robust}. An
$\eta$-\emph{robust} algorithm is one which maintains guarantees of accuracy
when given $\eta$-corrupted samples.

Both robustness and privacy demand that the output of some statistical method
not ``change too much'' when one or a few input samples are modified
arbitrarily.  This conceptual similarity has not gone unnoticed:
\cite{dwork2009differential} observe that ``robust statistical estimators
present an excellent starting point for differentially private estimators,'' and
recent works have even made good on this idea in high-dimensional settings,
taking inspiration from robust statistics to design private
mechanisms~\cite{liu2021robust, hayes2022learning}.  However, a general account
of which private algorithms can be made robust, or vice versa, remains an open
problem.
Here we take a small step by tackling:
\begin{center}
    \vspace{-0.2em}
\emph{Question 1: When are private algorithms (also) robust?}
\end{center}
\vspace{-0.2em}
While common wisdom holds that privacy and robustness are not formally comparable, we give a meta-theorem~\cref{thm:meta} which quantifies the degree of robustness that
private mechanisms exhibit: at a high level, mechanisms satisfying quantitatively-strong versions of DP are automatically robust.

This connection between robustness and privacy gives us the tools to investigate a second basic question, concerning tradeoffs among computational resources, privacy, and statistical accuracy.
\emph{Accuracy}-privacy tradeoffs appear in even the most elementary statistical
settings. For instance, estimating the mean of a $d$-dimensional Gaussian to
$\ell_2$ error $\alpha$ requires $\Theta(d/ {\alpha^2})$ samples non-privately,
but subject to $\e$-DP\footnote{See \cref{sec:auto-robust} for formal definitions.}
requires (roughly) $\Theta({d}/{\alpha^2} + {d}/{(\alpha
\e)})$ samples \cite{bun2019private}. While important in their own right, these two-way
tradeoffs are not the whole story.  In numerous cases, even including the
Gaussian mean estimation problem, \emph{computationally efficient algorithms
which achieve the optimal privacy-accuracy tradeoffs are not known}. This
brings up the question:
\begin{center}
    \vspace{-0.2em}
    \emph{Question 2: Does requiring differential privacy introduce computational barriers in statistics?}
\end{center}
\vspace{-0.2em}

While computational barriers to efficient private algorithms are known in some
settings \cite{gupta2013privately, ullman2016answering, bun2020computational},
these apply only to algorithms with \emph{worst-case} accuracy guarantees -- in
statistical settings we are asking only for \emph{average-case} accuracy
guarantees (although privacy should still hold with respect to all possible
datasets).

An archetypal problem for which computational barriers arise when
estimators are required to satisfy criteria beyond accuracy is that of
\emph{sparse mean estimation}.
The goal there is to estimate a $k$-sparse vector $\mu \in \R^d$ using
independent samples from a distribution with mean $\mu$.  With no requirements on
privacy or robustness, this can be accomplished with $O(k \log d)$ samples in
polynomial time via simple thresholding-based estimators.
In exponential time, it is possible to retain $O(k \log d)$ sample complexity
and satisfy privacy and robustness (as we show in this paper), but in polynomial
time $\Omega(k^2)$ samples are required (under a variant of the planted clique
conjecture) just to satisfy robustness \cite{brennan2020reducibility}.  Given
the outlined connection between robustness and privacy, the existence of an
information-computation gap for \emph{robust} sparse mean estimation naturally
leads to:
\begin{center}
    \vspace{-0.2em}
    \emph{Question 3: How many samples do poly-time private algorithms for sparse mean estimation require?}
\end{center}
\vspace{-0.2em}

\subsection{Our Contributions}
\label{sec:contributions}
We make three main contributions:
(a) a meta-theorem characterizing robustness of private mechanisms;
(b) a case study of sparse mean estimation, including a computational lower
bound arising from robustness, and a new Sum-of-Squares-based algorithm whose
sample complexity matches that lower bound in parameter regimes where no such
algorithms were previously known;
and (c) a collection of computational and information-theoretic lower bounds for
private mechanisms, inherited from lower bounds for robust algorithms.

\begin{table}[t!]
    \begin{threeparttable}[b]
    \caption{\textbf{Algorithms for sparse mean estimation.}
    In the input column, $\cN$ signifies that the algorithm takes i.i.d. samples
    from $\cN(\mu, I)$, and $\Sigma\preceq I$ --- samples from a
    distribution with bounded covariance.  The sample complexity column hides
    polylogarithmic factors in the ambient dimension $d$ and a priori mean bound
    (in $\ell_2$) $R$. In the auto-robust column we indicate whether our
    meta-theorem~(\cref{thm:meta}) implies that the algorithm is robust to
        corruptions of an $\eta = 1/\poly \log(d,R)$ fraction of the samples.}
    \label{table:algorithm-results}
    \centering
    \begin{tabular}{lllcc}
      \toprule
      Algorithm & Input & Runtime & Sample Complexity & Auto-robust \\
      \midrule
      \texttt{Hypothesis Sel.}~\cite{bun2019private}& $\cN$ & exp &
            $\frac{k + \log(1/\beta)}{\alpha\epsilon} +
            \frac{k + \log(1/\beta)}{\alpha^2}$ & \cmark \\
      \texttt{Subset Sel.}~[\cref{thm:sparse-exp-bcov}] & $\Sigma \preceq I$ & exp &
            $\frac{k + \log(1/\beta)}{\alpha^2\epsilon}$ & \cmark \\
      \texttt{SoS}~[\cref{thm:sparse-alg}] & $\cN$ & poly &
            $\frac{k^2 + \log(1/\beta)}{\alpha^2\epsilon}$ & \cmark \\
      \texttt{Threshold}~[\cref{thm:sparse-alg-peeling}] & $\cN$ & poly &
            $\frac{k^2 \log(1/\beta)}{\alpha^2\epsilon}$ & \xmark \\
      \texttt{Peeling}~\cite{cai2021cost}~\tnote{1} & $\cN$ & poly &
            $R\left(\frac{k^{1.5} \log(1/\beta)}{\alpha\epsilon} +
            \frac{k^{1.5}\log(1/\beta)}{\alpha^2}\right)$ & \xmark \\
      \bottomrule
    \end{tabular}
    \begin{tablenotes}
      \item [1]
      \texttt{Peeling} is stated only for $(\epsilon,\delta)$-DP
      in~\cite{cai2021cost}. We use a (mildly) modified $\epsilon$-DP version,
      which we formally state in~\cref{sec:restate-cost-of-privacy}.
    \end{tablenotes}
    \end{threeparttable}
  \end{table}

\paragraph{Robustness of Optimal Private Mechanisms}
Our first contribution is a simple but useful observation: \emph{mechanisms with
strong \textbf{group privacy} guarantees are automatically robust!}
By ``strong group privacy,'' we mean guarantees strong enough to retain
privacy and accuracy when a constant fraction of individuals in the dataset are
exchanged with others.
We capture this in \cref{thm:meta}.

In spite of its simplicity, robustness of strongly-group-private mechanisms has
significant consequences for simultaneously robust and private mechanisms, a
topic of much recent interest \cite{kothari2021private,
liu2021differential, liu2021robust, usynin2022collaborative, chhor2022robust},
because mechanisms with optimal privacy-accuracy tradeoffs often automatically
satisfy strong group privacy.

The privacy guarantees needed for automatic robustness rely on
quantitatively-strong forms of differential privacy: either so-called
\emph{pure} DP, or \emph{approximate} DP, but with the additive error parameter
$\delta$ taken exponentially small, and high success probability.
(By contrast, typical convention in the privacy literature is to take $\delta$
only polynomially small in other parameters.)
For this reason, we focus primarily on mechanisms satisfying pure DP and
achieving high success probability.

\paragraph{Sparse Mean Estimation}
Recall that in sparse mean estimation the goal is to estimate a $k$-sparse
vector $\mu \in \R^d$ to $\ell_2$-error $\alpha$, succeeding with probability
$1-\beta$, using i.i.d. samples from a distribution with mean $\mu$.
We contribute (a) new exponential-time $\e$-DP mechanisms using $O(k \log d)$
samples, (b) evidence that poly-time DP algorithms with high success
probabilities require $\tilde{\Omega}(k^2)$ samples, and (c) new poly-time
algorithms using $\tilde{O}(k^2)$ samples (in certain parameter regimes).
One of these algorithms, using the SoS exponential mechanism of
\cite{hopkins2021efficient}, is our main technical contribution.

\emph{Information-theoretic bounds:}
First, as a baseline, we study sparse mean estimation without worrying about
running time.
We show that even subject to both privacy and robustness, $O(k \log d)$ samples
suffice, with sample complexities differing in their dependence on $\alpha$ and
$\epsilon$ on between the cases that the underlying distribution is assumed to
be Gaussian versus allowing for heavy-tailed distributions (assuming only
bounded covariance), as in the case of non-sparse mean estimation.
See the estimators \texttt{Hypothesis Sel} and \texttt{Subset Sel}
in~\cref{table:algorithm-results}, and ~\cref{sec:sparse-inf-theory} for formal statements.

\emph{Computational lower bound:}
Next, we give evidence for a privacy-samples-success probability tradeoff for
efficient private algorithms: a private algorithm for sparse mean estimation
with high success probability will satisfy strong group privacy, and hence
robustness, but efficient and robust algorithms for sparse mean estimation
require $\Omega(k^2)$ samples (assuming the planted clique conjecture).
\begin{corollary}[of \cref{thm:meta} and \cite{brennan2020reducibility}, Theorem
3.1]\label{thm:sparse-lb}
Assume the \emph{secret-leakage planted clique} conjecture
\cite{brennan2020reducibility}.
For polynomially-related $n,k,$ and $d$, with $k = o(\sqrt{d})$, $\beta \in
(0,1)$, and $\alpha, \e > 0$, assume $\beta \leq 2^{-\epsilon \sqrt{n}}$.
Let $m(\beta, n, \epsilon)$ be the greatest value
less than ${\log(1/\beta) }/{(\epsilon n)}$ in the set $\{n^{-o(1)} \} \cup
\{ n^{-1/(2t)} \, : \, t \in \N, t \geq 1 \}$.
Every polynomial-time, $\e$-DP algorithm which, for any $k$-sparse $\mu \in
\R^d$ with $\|\mu\| \leq \poly(d)$ can take $n$ samples from $\cN(\mu,I)$ and
return $\hat{\mu}$ such that $\|\hat{\mu} - \mu\| \leq \alpha$ with
probability $1-\beta$ requires $n \geq \tfrac{k^2
m^2}{\alpha^4 \cdot \poly \log(d,\tfrac{1}{\epsilon},\tfrac{1}{\alpha})}$.
\end{corollary}

Pretending $m = \log(1/\beta) / (\epsilon n)$ and ignoring logarithmic factors in $d,\frac 1 \epsilon, \frac 1 \alpha$, the lower bound says $n \gtrsim \Paren{ \frac{k^2}{\alpha^2} \cdot \Paren {\frac{\log(1/\beta)}{\alpha \epsilon}}^2 }^{1/3}$
samples are required by efficient $\e$-DP algorithms for sparse mean estimation.
This is the geometric mean of three terms: $\tfrac{k^2}{\alpha^2}$ and $\tfrac{\log(1/\beta)}{\alpha \e}$ (twice).
We conjecture that the max of these is actually a lower bound:
$n \gtrsim \frac{k^2}{\alpha^2} + \frac{\log(1/\beta)}{\alpha \epsilon}$.

Even if this stronger lower bound were true, \emph{existing efficient algorithms
for sparse mean estimation would not match it}. The state-of-the-art for sparse
$\e$-DP mean estimation using techniques in the literature is a modification (to
achieve pure DP) of an algorithm by \cite{cai2021cost}, whose sample complexity
scales with $k^{1.5} \log(1/\beta)$ (see \cref{sec:restate-cost-of-privacy}).

\emph{Efficient algorithms -- SoS: }
Are there polynomial-time algorithms which match the tradeoff from \cref{thm:sparse-lb}, or the stronger conjectured one above?
Our main algorithmic contribution is a new algorithm for sparse mean estimation
which matches the tradeoff of \cref{thm:sparse-lb} (up to polylog factors) under the conditions:
 (1) $\alpha,\epsilon \geq 1/\poly \log (d)$, (2) $k^2 \approx \log(1/\beta)$, and (3) $k \geq d^{0.4}$.
Of these conditions, (2) could be removed if the stronger conjectural lower bound above were true, while we believe that (1) and (3) are shortcomings of our algorithm, and might be removable.
We are not aware of any previous efficient private algorithm which matches the above conjectured tradeoff for any simultaneously super-constant $k$ and $\log(1/\beta)$.

\begin{theorem}\label{thm:sparse-alg}
  %[informal version of \cref{thm:sos_sparse_mean_algo}]
There exists $C>0$ such that for every $\e, R > 0$, $\alpha,\beta \in (0,1)$,
and large-enough $d,k \in \N$ such that $k\geq d^{0.4}$, there is a
polynomial-time $\e$-DP algorithm \texttt{SoS} with the following guarantees.
For every $k$-sparse $\mu \in \R^d$ with $\|\mu\| \leq R$, given
$\eta$-corrupted samples $X_1,\ldots,X_n \sim \cN(\mu,I)$, with probability at
least $1-\beta$, the algorithm outputs $\hat\mu \in \R^d$ such that $\|\mu -
\hat\mu\| \leq \alpha +O\left(\sqrt{(\log (Rd))^C\eta}\right)$, so long as $n
\gg (\log (Rd))^{C} \cdot \frac{ k^2 + \log(1/\beta) + \log \log R
}{\alpha^2\e}.$
\end{theorem}

Our algorithm employs the Sum-of-Squares exponential mechanism invented by
\cite{hopkins2021efficient} for a private mean estimation algorithm, but
adapting this approach to the sparse setting requires overcoming several
technical roadblocks (see~\cref{sec:sparse-mean} and
\cref{sec:strongly-private-sparse}).
Our approach hits a technical obstacle related to the volume of the
$d$-dimensional $\ell_1$ ball when $k \ll d^{0.4}$, and we leave as
an open problem to match or approach the tradeoff in \cref{thm:sparse-lb} for a
wider range of parameters.

\emph{Linear-time coordinate selection:}
Finally, what can be accomplished with a ``truly efficient'' algorithm -- one which does not require solving large semidefinite programs as in the SoS exponential mechanism?
The state-of-the-art private algorithm for sparse mean estimation, of \cite{cai2021cost}, is a simple iterative coordinate-selection procedure.
With a minor modification to achieve a pure DP guarantee, that algorithm uses $O_{R}(k^{1.5} \log d)$ to estimate $k$-sparse mean vectors $\mu \in \R^d$ with $\|\mu\| \leq R$, with probability $1-\beta$.
However, the $O_R(\cdot)$ hides a \emph{linear}, rather than logarithmic, dependence on $R$, which is very costly even for moderately-large values of $R$!

While it is now well understood how to obtain logarithmic-in-$R$ sample complexity for non-sparse mean estimation, standard approaches introduce a linear dependence on ambient dimension $d$.
We improve over the algorithm of \cite{cai2021cost} while maintaining linear running time by designing a simple thresholding procedure for estimating the support of $\mu$ requiring a number of samples which is \emph{independent of $R$}.
Once the support is known, the ambient dimension of the problem can be reduced
from $d$ to $k$, and a off-the-shelf private non-sparse mean estimation
algorithm can be run. A formal statement is given
in~\cref{sec:experimental-details}.
We demonstrate with experiments on synthetic data (\cref{sec:sparse-mean}) that
the coordinate-selection procedure \texttt{Threshold}, is substantially more
accurate than that used by prior state of the art~\cite{cai2021cost}.

By contrast to \texttt{SoS}, the linear-time algorithm \texttt{Threshold} has
sample complexity that scales with $k^2 \log(1/\beta)$ rather than $k^2 +
\log(1/\beta)$. While this difference might appear minor at first, it has
significant consequences: the linear-time algorithm cannot match the
computational lower bound in \cref{thm:sparse-lb}, and it does not have
strong-enough privacy guarantees to be robust via \cref{thm:meta}. Note that our
results do not preclude other tradeoffs between $k$ and $\log(1/\beta)$. For
instance, we leave it as an exciting open problem to design an efficient
algorithm with sample complexity scaling with $k\log(1/\beta)$.

\paragraph{Information-Computation Gaps in Private Statistics}
A problem in private statistics has an \emph{information-computation} gap if the
accuracy-privacy tradeoffs achieved by optimal (exponential-time) mechanisms for
that problem cannot be achieved by polynomial-time algorithms.  Our connection
between privacy and robustness can be used to give evidence for several
information-computation gaps in private statistics beyond sparse mean estimation.

\emph{Gaussian mean estimation} (\cref{cor:gaussian-lb}): Gaussian mean
estimation, where the goal is to estimate $\mu \in \R^d$ given i.i.d. samples
from $\cN(\mu,I)$, is arguably even simpler than sparse mean estimation. We give
evidence for an information-computation gap: polynomial-time $\e$-DP algorithms
obtaining accuracy $\alpha$ with probability at least $1-\beta$ require $n
\geq \log(1/\beta) \cdot (\log(1/\alpha))^{1/2 - o(1)} / (\alpha \epsilon))$
samples, unless there exist robust polynomial-time algorithms for Gaussian mean
estimation which would contradict known statistical query lower bounds
\cite{diakonikolas2017statistical} -- this is a $\log(1/\alpha)^{1/2 - o(1)}$
factor gap.

\emph{Learning parities} (\cref{corr:learning-parities}):
We also consider one of the most fundamental \emph{supervised} learning
problems: privately PAC learning parity functions from labeled examples. For
each $S \subseteq [n]$, we can define a parity function $f_S \, : \, \{ \pm
1\}^n \rightarrow \{ \pm 1\}$ by $f_S(x) = \prod_{i \in S} x_i$. The goal is to
take labeled examples $(x,y)$ drawn from some distribution $D$ and find a
parity function $f_S$ such that $\Pr_{(x,y) \sim D}(f_S(x) = y) = 1$, assuming one exists.
While polynomial-time private algorithms for learning parities are
known~\cite{kasiviswanathan2011can}, we show, via \cref{thm:meta} that the
failure probabilities of any such algorithms must be larger than what can be
achieved in exponential time, or else $RP = NP$.

\paragraph{Information-Theoretic Lower Bounds in Private Statistics}
Finally, we show that the connection to robustness can provide information-theoretic lower bounds for private mechanisms. As an
example, we study private covariance testing (\cref{cor:covariance_testing}),
where the goal is to take samples
from $\cN(0,\Sigma)$ and detect whether $\Sigma = I \in \R^{d \times d}$ or if
$\|\Sigma - I\|_F \geq \gamma$.  Appealing to the lower bound of
\cite{diakonikolas2021sample} for robust covariance testing, we give a lower
bound for private covariance testing, showing that $\Omega(d^2)$ samples are
required by private algorithms with high success probabilities, while $O(d)$ suffice non-privately.

Formal statements for the results on Gaussian mean estimation, learning
parities, and covariance testing can be found in~\cref{sec:lb}.

\subsection{Related Work}\label{sec:related_work}
\emph{Privacy and robustness.} As mentioned in~\cref{sec:intro}, there is a
rich history of connections between DP and robustness, starting from the
propose-test-release (PTR) framework of Dwork and
Lei~\cite{dwork2009differential}. Building on top of PTR, a number of recent
works tackle high-dimensional statistics problems by leveraging robust
primitives~\cite{brown2021covariance, liu2021differential}, themselves inspired by a recent revolution in high-dimensional robust statistics \cite{diakonikolas2019recent}.
On the flip side, private algorithms for certain problems have been shown to
``automatically'' exhibit a small amount of
robustness~\cite{thakurta2013differentially, hopkins2021efficient}.

\emph{Sparse mean estimation.} Without privacy or robustness requirements,
it is a folklore result that the truncated empirical mean achieves the
information-theoretically optimal rate.
In the approximate DP case,~\cite{thakurta2013differentially} show that the
stability of LASSO can be leveraged for private support selection and private
sparse regression via the sample-and-aggregate
framework~\cite{nissim2007smooth}. Cai, Wang, and Zhang~\cite{cai2021cost} show
information-theoretic lower bounds for approximate DP (based on tracing
attacks~\cite{homer2008resolving}) and computationally
efficient algorithms that match those bounds with constant probability under
additional assumptions on the $\ell_\infty$ norm of the mean $\mu$.
In the presence of $\eta$-corruptions of the
samples,~\cite{balakrishnan2017computationally} gives an $\tilde
O\left(k^2\log(d)/\eta^2\right)$-sample algorithm matching SQ lower bounds
from~\cite{diakonikolas2017statistical}.

\emph{Computational Roadblocks to Privacy.} Several prior works investigate computational roadblocks to privacy arising from cryptographic considerations, e.g. \cite{ullman2016answering,ullman2011pcps}.
The hard problem instances constructed in such works have a worst-case flavor, while we are interested in computational hardness for typical datasets/those drawn i.i.d. from an underlying probability distribution.

\emph{Sum-of-Squares Method.} The SoS method for algorithm design in
high-dimensional statistics has led to a number recent of algorithmic advances
-- see the survey \cite{raghavendra2018high}. \cite{hopkins2021efficient}, which
introduces the SoS exponential mechanism, is most closely related,
and provides the foundations for \cref{thm:sparse-alg}.

\emph{Lower bounds for private mechanisms} There are multiple works leveraging
group privacy to derive lower bounds for private
algorithms, both for pure and approximate DP~\cite{hardt2010geometry,
de2012lower, steinke2015between, bun2016concentrated}. They are all
information-theoretic in nature, while we also provide computational hardness
results. Additionally, to the best of our knowledge, there are no existing work
relating lower bounds for privacy to ones for robustness.

%% file: sections/prelims.tex
\section{Preliminaries}\label{sec:prelim}
We first introduce central definitions and results from differential privacy
that we use throughout the paper. Next, we give background on the notion of
robustness we work with. Finally, we briefly overview the central ideas in the
Sum-of-Squares (SoS) tools we use.

\paragraph{Differential Privacy}
We start by formally defining the notion of differential privacy.
\begin{definition}[(Approximate) Differential Privacy]
Let $\cX$ be a set and $\cX^* = \{(X_1,\ldots,X_n) \, : \, n \in \N, X_i \in
\cX\}$ be all possible datasets over $\cX$. For $\e,\delta > 0$, a (randomized)
map $M \, : \, \cX^* \rightarrow \cO$ (where $\cO$ is an \emph{output} set) is
$(\e,\delta)$-DP if for every $(X_1,\ldots,X_n), (X_1',\ldots,X_n') \in \cX^*$
such that $X_i = X_i'$ except for a single index $i$ and for every subset $S
\subseteq \cO$, $\Pr( M(X_1,\ldots,X_n) \in S) \leq e^{\e} \Pr
(M(X_1',\ldots,X_n') \in S) + \delta$.
\end{definition}

The special case of $\delta=0$ is referred to as \emph{pure} DP and will be the
main focus of our work. Throughout the paper, we refer to datasets that differ
in one entry as \emph{neighboring}.

We freely use basic primitives in private algorithm design like composition of
private mechanisms, which we state for completeness.
\begin{lemma}[(Basic) Composition~\cite{dwork2006calibrating}]
    Suppose that we have private mechanisms $M_1,\ldots, M_n$ where $M_i$ is
    $\e_i$-DP. Then an adaptive composition $M$ of $\{M_1,\ldots, M_n\}$ is
    $\sum_{i\leq n} \e_i$-DP.
\end{lemma}
We also freely use basic DP mechanism like the Laplace mechanism defined below.
\begin{lemma}[Laplace mechanism~\cite{dwork2006calibrating}]\label{lem:laplace}
    Let $f\colon \cX^n\to \R^d$ such that $\max_{X,X'\
    \text{neighbors}}\|f(X)-f(X')\|_1=\Delta$. Then the mechanism
    $M:\cX^n\to\R^d$ defined as $M(X) = f(X) + L$, where $L\sim
    \text{Lap}\left(0, \frac{\e}{\Delta}\right)$ is sampled from the Laplace distribution, is
    $\e$-DP.
\end{lemma}
The quantity $\Delta$ in \cref{lem:laplace} is referred to as the $\ell_1$
sensitivity of $f$. A central part of our design of private mean estimation
algorithms revolves around using functions with low sensitivity as primitives
for our private estimators.

We make extensive use of the exponential mechanism~\cite{mcsherry2007mechanism}.
It is a technique to privately select the (approximately) ``best'' object in a
universe $\cH$ according to a \emph{score function} $S$, which measures
the ``goodness'' of a given object. An important property of the score function
is its \emph{sensitivity} $\Delta_S$, defined as
$\max_{h\in\cH, \text{neighboring}\ D, D'} |S(D, h) - S(D', h)|$.
Given a score function $S$ with sensitivity $\Delta_S$ and a privacy parameter
$\e$, the exponential mechanism samples an object $h$ with probability
proportional to $\exp\left(\e/(2\Delta_S)\cdot S(h)\right)$. The exponential mechanism comes with the following privacy
and utility guarantees.
\begin{theorem}[{{\cite{mcsherry2007mechanism}}}]
For a dataset $X$ and a score function $S\colon \cX^n \times \cH \to \R$, the
exponential mechanism $M$ on the score function $S$ is $\e$-DP and with
probability at least $1-\beta$ outputs an object such that
\[
    S(M(X)) \geq OPT(X) -
    \frac{2\Delta_S}{\e}\left(\log\left(\frac{|\cH|}{|\cH^*|} +
     \log(1/\beta)\right)\right),
\]
where $\cH^*$ is the set of objects achieving score $OPT$.
\end{theorem}

\paragraph{Robustness} There are a variety of models for the adversary in the
agnostic setting. We work with the case of \emph{adaptive} adversaries:
\begin{definition}[$\eta$-contamination model]\label{def:robustness}
    In the $\eta$-contamination model, given a ``clean'' distribution $D$, to
    draw $n$ ($\eta$-contaminated) samples from $D$, first draw
    $X_1',\ldots,X_n'\sim D$, and then output any $\{X_1,\ldots, X_n\}$ such
    that $X_i=X_i'$ for at least $(1-\eta)n$ choices of $i$.
\end{definition}

Another popular model in the literature is that of an \emph{oblivious}
adversary who provides a distribution that is close in total variation distance
to the original one. The name comes from the fact that, unlike in
\cref{def:robustness}, the adversary is not allowed to inspect the samples and
adaptively decide what samples to add and remove.

\paragraph{Sum of Squares} In~\cref{sec:strongly-private-sparse} we use the SoS
exponential mechanism in a black-box fashion. For a derivation of it, as well as
examples, see~\cite{hopkins2021efficient}. Informally, we use SoS as a
proofs-to-algorithms paradigm which automatically ``simple'' (expressible within
the SoS proof system) proofs into polynomial time algorithms.  For an overview
of SoS, see e.g.~\cite{barak2014sum}.

%% file: sections/metatheorem.tex
\section{Automatic Robustness Meta-Theorem and Private Robust
Mechanisms}\label{sec:auto-robust} 

In this section we state and prove our meta-theorem on automatic robustness of
private algorithms.

\begin{theorem}[Automatic Robustness Meta-Theorem] \label{thm:meta}
    Let $M \, : \, \cX^* \rightarrow \cO$ be an $(\epsilon,\delta)$-private map
    from datasets $\cX^*$ to outputs $\cO$.  For every dataset $X_1,\ldots,X_n$,
    let $G_{X_1,\ldots,X_n} \subseteq \cO$ be a set of \emph{good} outputs.
    Suppose that $M(X_1,\ldots,X_n) \in G_{X_1,\ldots,X_n}$ with probability at
    least $1-\beta$ for some $\beta = \beta(n)$.  Then, for every $n \in \N$, on
    $n$-element datasets $M$ is \emph{robust} to adversarial corruption of any
    $\eta(n)$-fraction of inputs, where 
    \[
        \eta(n) = O \Paren{ \min \Paren{ \frac{\log 1/\beta}{\e n}, \frac{\log
        1/\delta}{\e n + \log n} } }\mcom
    \]
    meaning that for every $X_1,\ldots,X_n$ and $X_1',\ldots,X_n'$ differing on
    only $\eta n$ elements, $M(X_1',\ldots,X_n') \in G_{X_1,\ldots,X_n}$ with
    probability at least $1- \beta^{\Omega(1)}$.
    
\end{theorem}
\begin{proof}[Proof of \cref{thm:meta}]
  Consider $\eta n$ intermediate datasets $(X_1,\ldots,X_n) =
  \mathbf{X}_0,\ldots,\mathbf{X}_{\eta n} = (X_1',\ldots,X_n')$, where a single
  coordinate $X_i$ is modified in passing from $\mathbf{X}_j$ to
  $\mathbf{X}_{j+1}$.  Let $p_j = \Pr(M(\mathbf{X}_j) \in G_{X_1,\ldots,X_n})$.
  Then we have the following recurrence for $(1-p_j)$: 
  \[
    (1-p_j) \leq e^{\e} (1-p_{j-1}) + \delta \text{ for $j \geq 1$, and }  (1-p_0) \leq \beta\mcom
  \]
  from which we obtain $(1-p_{\eta n}) \leq e^{\e \eta n} (\beta + \eta n \delta)$.
  The conclusion follows.
\end{proof}
An analogous statement for concentrated DP~\cite{bun2016concentrated} is
presented in~\cref{sec:zCDP}.

\cref{thm:meta} can be applied broadly to show that optimal private
mechanisms are automatically robust, frequently even with optimal dependence of the lost accuracy on the rate of corruption. This breadth is possible
because, for many statistical problems, \emph{private mechanisms obtaining
information-theoretically optimal privacy-accuracy tradeoffs automatically have
the strong group privacy guarantees needed to apply \cref{thm:meta}.}
This is because strong group privacy for a mechanism $M$ is implied by two other
desirable properties of private mechanisms: (1) $M$ satisfies \emph{pure} DP
(or, $(\e,\delta)$-DP for small choices of $\delta$), and (2) $M$
produces accurate results \emph{with high probability} over the randomness used
internally by the mechanism.  We give two examples of this phenomenon below. 

\emph{On the success probabilities of private algorithms: } Before turning to
examples, we observe that \cref{thm:meta} only gives robustness to a constant fraction of corrupted samples for 
private algorithms which have very high probability of succeeding -- to obtain
$\eta \geq \Omega(1)$ requires $\beta \leq 2^{-\Omega(n)}$.  In most work on
randomized algorithm design, the difference between succeeding with probability
$2/3$ versus $1-\beta$ for small $\beta$ can be treated as an afterthought,
because algorithms can be repeated to amplify success probability.  \emph{But
this kind of naive repetition causes privacy leakage!} 

In spite of this, optimal private mechanisms in statistics often do succeed with
high probability, using more sophisticated approaches than naive repetition: in
fact, high success probability is generally implied by $M$'s outputs having
(asymptotically) \emph{optimal confidence intervals}.  Our work points to a need
for algorithm designers to focus on the confidence intervals/success
probabilities of private algorithms: the payoff is robustness for free. 

\emph{On Black-Box Robustification of Optimal Private Mechanisms: }
We observe that \cref{thm:meta} can be used to automatically obtain robust and private mechanisms from private ones with high success probabilities.
For instance, \emph{(1) Hypothesis selection:} The private hypothesis selection procedure of \cite{bun2019private}, already proved robust against non-adaptive adversaries, is additionally robust to corruptions made by adaptive adversaries.
\emph{(2) Affine-invariant mean estimation:}
\cite{brown2021covariance,liu2021differential} study mean-estimation mechanisms
which provide error guarantees in the \emph{Mahalanobis distance}
$\|\Sigma^{-1/2}(\hat{\mu} - \mu)\|$ given samples from $\cN(\mu,\Sigma)$.
\cite{brown2021covariance} give a private mechanism for this problem with high
success probability, and \cite{liu2021differential} give a simultaneously
private and robust mechanism. 
In either case, because both mechanisms provide strong-enough privacy guarantees
to apply \cref{thm:meta}, robustness can be obtained in a black-box
fashion knowing only the privacy guarantees. 

\emph{From \cref{thm:meta} to Lower Bounds for Private Statistics: } As
we discussed in \cref{sec:contributions}, we use \cref{thm:meta}
to prove both computational and information-theoretic sample-complexity lower
bounds for private algorithms for sparse mean estimation
(\cref{thm:sparse-lb}), non-sparse mean estimation, learning parities,
and covariance testing. 
We prove \cref{thm:sparse-lb} in the next section, and defer the
remaining statements and proofs of lower bounds to \cref{sec:lb}.  

%% file: sections/hardness.tex
\section{Lower Bounds in Private Statistics}
\label{sec:lb}

In this section, we state and prove theorems giving information-computation gaps
for private algorithms for Gaussian mean estimation and learning parities, and
we give an information-theoretic lower bound for private covariance testing,
again via \cref{thm:meta}. 

\subsection{Private Gaussian Mean Estimation}
Consider the task of privately estimating the mean of class of spherical
Gaussian distributions, and the following conjecture, supported by statistical
query lower bounds \cite{diakonikolas2017statistical}. 

\begin{conjecture}\label{conj:gaussian-hardness}
  There exists $c > 0$, no $\poly(n,d,1/\eta)$-time algorithm $\eta$-robustly estimates the mean of an
    unknown $d$-dimensional spherical Gaussian $\cN(\mu,I)$ with, $\|\mu\| \leq \poly(d)$, from independent samples to $\ell_2$ error
    $\eta (\log 1/\eta)^{1/2 - \Omega(1)}$, so long as $\eta \geq 2^{-d^{c}}$.
\end{conjecture}

(By contrast, $O(\eta)$ error is achievable in exponential time.) We prove the following corollary:

\begin{corollary}\label{cor:gaussian-lb}
    Assume \cref{conj:gaussian-hardness}.
    Every $\poly(n,d,1/\e)$-time $\e$-DP algorithm which takes i.i.d. samples from $\cN(\mu,I)$ for $\mu \in \R^d$ and outputs a vector $\hat{\mu}$ such that $\|\hat{\mu}- \mu\| \leq \alpha$ with probability $1-\beta$ for all $\|\mu\| \leq \poly(d)$ requires $n \geq \min \{ \tfrac{\log(1/\beta)}{\epsilon \alpha} \cdot (\log(1/\alpha))^{1/2 - o(1)}, \tfrac{2^{d^c}\log(1/\beta)}{\e} \}$, where $c>0$ is a universal constant.
\end{corollary}
To see that \cref{cor:gaussian-lb} captures an information-computation gap, recall that in exponential time it is possible to perform $\e$-DP Gaussian mean estimation for $\|\mu \| \leq \poly(d)$ with
\[
    n = \Theta \Paren{ \frac{d + \log(1/\beta)}{\alpha^2} + \frac{d + \log(1/\beta)}{\alpha \epsilon} + \frac{d \log d + \log(1/\beta)}{\epsilon} }
\]
samples \cite{bun2019private}.
Unless the SQ lower bounds of \cite{diakonikolas2017statistical} are broken by some polynomial-time algorithm, this error rate can be matched, at best, up to logarithmic factors by polynomial time algorithms.
\begin{proof}[Proof of \cref{cor:gaussian-lb}]
  A polynomial-time private algorithm with the hypothesized guarantees would be
  $\eta = \log(1/\beta)/ (\e n)$-robust by \cref{thm:meta}. So long as
  $\log(1/\beta) / \e n = \eta \geq 2^{-d^c}$ this satisfies the hypotheses of
  \cref{conj:gaussian-hardness}, so it can't estimate $\mu$ to error
  $\eta(\log(1/\eta))^{1/2 - o(1)}$; the conclusion follows by substituting
  $\log(1/\beta)/(\e n)$ for $\eta$. 
\end{proof}

\subsection{Privately Learning Parities}\label{sec:parities}
We now turn to a fundamental supervised problem - PAC learning parities.
In the absence of noise, Gaussian elimination provides a polynomial-time
sample-efficient algorithm for learning parities.

Learning under adversarial label noise is known as agnostic
learning~\cite{kearns1994toward}. Learning parities in the agnostic case, unlike
in the noiseless case, is notoriously difficult; in the proper case, the problem has been
shown to be NP-hard~\cite{haastad2001some}, even for getting accuracy of $1/2 +
\epsilon$ for any $\epsilon > 0$. With this in mind, our goal is to probe where
does private parity learning lie computationally.

More formally, let \texttt{PARITY} be the class of functions $f_S \, : \, \{ \pm
1\}^d \rightarrow \{ \pm 1\}$ defined by $f_S(x) = \prod_{i \in S} x_i$ for each
$S \subseteq [n]$. 
An algorithm PAC learns \texttt{PARITY} if, for every $S$, given $n$ samples $(x,f_S(x))$ where $x$ comes from some distribution $D$ on $\{ \pm 1\}^d$, it finds $T$ such that $\Pr_{x \sim D} (f_S(x) \neq f_T(x)) \leq \alpha$.
\cite{kasiviswanathan2011can} show that 
\begin{equation}
    n=O\left(\frac{d\log (1/\beta)}{\epsilon \alpha}\right)
\end{equation}
suffice to privately PAC learn \texttt{PARITY} in polynomial
time~\cite[Theorem 4.4]{kasiviswanathan2011can}, succeeding with probability $1-\beta$.
In contrast, allowing exponential time, only
\begin{equation}
    n=O\left(\frac{d+\log (1/\beta)}{\epsilon \alpha}\right)
\end{equation}
suffice.
From~\cref{thm:meta} we obtain the result that this gap cannot be closed
with polynomial-time private algorithms, unless $RP=NP$:
\begin{corollary}\label{corr:learning-parities}
    Suppose $RP \neq NP$.
    Then every polynomial-time $\epsilon$-DP algorithm which for any $\beta > 0$ can PAC-learn $d$-variable \texttt{PARITY}s to accuracy $\alpha$, succeeding with probability $1-\beta$, requires $n \geq \omega \Paren{ \frac{\log(1/\beta)}{\e \alpha} }$ samples.
\end{corollary}
\begin{proof}
     Suppose otherwise, that there exists a private PAC learner using $n \leq O(\log(1/\beta) / (\alpha \e))$ samples.
     By Theorem~\ref{thm:meta}, such an algorithm is also $\eta = \frac{\log(1/\beta)}{\e n}$-robust, with $\eta \geq \Omega(\alpha)$.
     We claim that it can be used to distinguish $1- \Omega(\alpha)$-satisfiable instances of XOR-SAT from $1/2+O(\alpha)$-satisfiable ones; by \cite{haastad2001some} this is NP-hard for every constant $\alpha > 0$.

     Given a $d$-variable instance $\phi$ of XOR-SAT, let $D$ be the following distribution on the hypercube.
     First, draw a clause $\prod_{i \in C} y_i = b_C$ from $\phi$ uniformly at random.
     Then let $c \in \{ \pm 1\}^d$ be given by $c_i = -1$ if $i \in C$ and otherwise $c_i = -1$.
     
     Suppose $\phi$ is $1-\alpha'$-satisfiable, by some $x \in \{ \pm 1\}^n$.
     Define $f_x(c) = \prod_{i \, : \, x_i = -1} c_i$.
     We claim that $\Pr_{C \sim \phi} (f_x(c) = b_C) \geq 1-\alpha'$.
     This holds because $f_x(c) = \prod_{i \, : \, x_i = -1} c_i = \prod_{i \, : \, c_i = -1} x_i = b_C$ if $C$ is a clause satisfied by $x$, and by hypothesis a $(1-\alpha')$-fraction of the clauses in $\phi$ are satisfied by $x$.

     Now, given samples $(c,f_x(c))$, our hypothesized private PAC learner returns, with probability $1-\beta$, some parity function $f_S$ such that $\Pr_{C \sim \phi} (f_x(c) = f_S(c)) \geq 1-\alpha'$.
     This means that $\Pr_{C \sim \phi}(f_S(c) = b_C) \geq 1-2\alpha'$.

     By $\eta$-robustness with $\eta \geq \Omega(\alpha)$ for $\alpha' \ll \alpha$, even given samples $(c,b_C)$, the learning algorithm returns such an $f_S$ with probability at least $0.99$, taking $\beta$ sufficiently small.
     But given $f_S$ it is easy to extract a $1-O(\alpha')$-satisfying assignment to $\phi$, which is NP hard.
\end{proof}

\subsection{Private Covariance Testing}\label{sec:cov_testing}

A fundamental question in high-dimensional statistics is that of
distributional property testing~\cite{batu2000testing}. Given independent
samples, the goal is to design efficient algorithms to test whether their
distribution satisfies a given property.

In this work, we are interested in a simple problem in that field - Gaussian
covariance testing. In particular, given samples $X_1,\ldots, X_n$ in $\R^d$
from $\cN(0, \Sigma)$ with unknown $\Sigma$, we want to \emph{privately}
determine whether $\Sigma=I$ or $\|\Sigma-I\|_F\geq \gamma$ with high
probability using as few samples as possible. 

In the non-private case, \cite{cai2013optimal} show that Gaussian covariance
testing can be achieved with $O\left(\frac{d}{\gamma^2}\right)$ samples,
significantly less than the $O\left(\frac{d^2}{\gamma^2}\right)$ samples
required to learn the distribution. Curiously, Diakonikolas and
Kane~\cite{diakonikolas2021sample} show that under $\eta$-corruptions for any constant $\eta > 0$ one needs
$\Omega(d^2)$ samples for testing. 

\cite{amin2019differentially} show that one needs
$O\left(\frac{d^2}{\gamma^2\epsilon}\right)$ samples to privately learn the
covariance of a Gaussian in $d$ dimensions. Thus, this is a natural upper bound
on the sample complexity of \emph{testing} the covariance.

We leverage \cref{thm:meta} to show that any private algorithm with
strong group privacy guarantees requires $\Omega(d^2)$ samples for covariance
testing. 

\begin{corollary}\label{cor:covariance_testing}
For every $C > 0$ there is $C' > 0$ such that
    every $\epsilon$-DP mechanism which takes independent samples $X_1, \ldots, X_n$ from $\cN(0, \Sigma)$
    for an unknown $\Sigma$ and distinguishes with probability at least
    $1-\beta$ whether $\Sigma=I$ or $\|\Sigma - I\|_F \geq 1/2$ takes $n \geq \Omega(\min( C \log(1/\beta)/\e, C' d^2 ))$ samples.
\end{corollary}
\begin{proof}
    \cite{diakonikolas2021sample} shows that for every $c > 0$ there is $c' > 0$ such that any $c$-robust algorithm for distinguishing $\Sigma = I$ from $\|\Sigma - I\| > 1/2$ requires at least $c' d^2$ samples.
    Any $\e$-private algorithm which accomplishes this task with probability at least $1-\beta$ will be $\log(1/\beta) / (\e n)$-robust, by Theorem~\ref{thm:meta}.
    If $n \leq C \log(1/\beta)/\e$, then this algorithm is $\Omega(1/C)$-robust.
    Hence, there is $C'$ such that it must use $n \geq C' d^2$ samples.
\end{proof}

%% file: sections/sparse_mean.tex
\section{Sparse Private Mean Estimation: Techniques}
\label{sec:sparse-mean}
In this section, we overview ideas which go into our results on sparse mean
estimation, starting from the information-theoretic results, and moving on to
computational barriers and polynomial-time algorithms.

\paragraph{Information-theoretic bounds}
In the absence of computational considerations, the landscape for
\emph{Gaussian} sparse mean estimation can be understood via standard tools: a
packing-based lower bound, and a matching (exponential-time) mechanism can be
constructed as a direct corollary of existing results in the literature on
private hypothesis selection~\cite{bun2019private}. For completeness, we carry
out those in \cref{sec:sparse-inf-theory}.

In the \emph{heavy-tailed} case, assuming only that the samples $X_1,\ldots,X_n$
are drawn from a distribution with $k$-sparse mean $\mu$ and with bounded
covariance, one can no longer construct a small cover of the set of possible
distributions; this precludes an approach as general as hypothesis selection
from directly applying. Instead, we design a mechanism which first selects a
subset of $k$ coordinates, then hands off to a non-sparse mean estimation
mechanism run just on those coordinates.

To select the coordinates, we use the exponential mechanism.
To define a score function, we take inspiration from recent ideas in
high-dimensional statistics~\cite{lugosi2019sub} using empirical quantiles of
univariate projections of the samples. The key idea for coordinate selection is
to restrict attention to projection in sparse directions. For samples
$X_1,\ldots,X_n$, we define the following score function on subsets of
coordinates $T \subseteq [d]$:
\begin{equation}\label{eq:sparse-score-fn-bcov}
    S(\{X_i\}_{i\leq n}, T; L) = \max_{v\in\mathbb{R}^k, \|v\|_2=1}
                    \sum_{i=1}^n \1\{ v^\top X_{iT} \geq L\},
\end{equation}
parametrized by a scalar threshold $L$\footnote{For a sample $X\in\R^d$, let
$X_T$ to be the projection onto the basis vectors with indices in the set
$T\subseteq [d]$.}. Then we sample $T$ from the distribution $\Pr(T) \propto
\exp(\e S(\{X_i\}_{i \leq n}, T ; L)$. Since $S$ has sensitivity
$\Delta(S)=1$,\footnote{That is, it changes by at most $1$ when any sample is
exchanged for another.} the resulting mechanism satisfies $\e$-DP. We show using
standard concentration tools that if $n \gg O(k\log(d) +
\log(1/\beta)/(\alpha^2\epsilon))$ then this mechanism identifies a subset $T$
containing all but $\alpha$ of the $\ell_2$-mass of $\mu$, with probability at
least $1-\beta$. To do so, we choose the threshold $L$ such that with high
probability all ``bad'' subsets containing a small portion of the $\ell_2$ mass of $\mu$ have ``low'' score
compared to ``good'' ones for which $\|\mu_T\|\approx\|\mu\|$; finally, we use bucketing to control the variance of
samples such that we can set $L$ as close to zero as possible without allowing
``bad'' subsets to achieve high score.

We then delegate the mean estimation on the candidate set of coordinates to any information-theoretically optimal (non-sparse) mean estimation mechanism.
For the matching lower bound, standard packing-based arguments suffice~(see
\cref{prop:sparse-lb-bcov} for a formal argument).

\paragraph{Computational barrier}
Now we turn to the landscape when we require polynomial-time.
Starting with barriers, we first show how to get~\cref{thm:sparse-lb} from our
meta-theorem.
\begin{proof}[Proof of~\cref{thm:sparse-lb}]
    First of all, since $n \geq \log(1/\beta)^2 / \e^2$, there exists a value $m$ in the given set.
    Now suppose a polynomial-time private algorithm exists tolerating $n \leq k^2 m^2 / (\alpha^4 \poly \log(d,1/\e,1/\alpha))$.
    By adjusting parameters $\beta, \e$ to artificially weaken the guarantees as necessary, we may assume that $m(\beta,n,\e) = \log(1/\beta) / (\e n)$.
    Now, Theorem 3.1 of \cite{brennan2020reducibility} shows that no algorithm for sparse mean estimation under these parameters can tolerate an $\eta$-fraction of adversarial corruptions with $k^2 \eta^2 / \alpha^4 \gtrsim n$, where $\gtrsim$ hides polylogarithmic factors.
    But from~\cref{thm:meta} our hypothesized private algorithm tolerates an
    $\eta = \log(1/\beta) / (\e n)$ fraction of adversarial corruptions, which
    is a contradiction.
\end{proof}

\paragraph{Overview of SoS Algorithm for Sparse Mean Estimation (\cref{thm:sparse-alg})}
% We turn to an overview of the algorithm behind~\cref{thm:sparse-alg}.
We remain informal in this section and defer mathematical rigor to \cref{sec:strongly-private-sparse}.
For now, let $\alpha = \Theta(1)$.
A standard trick reduces from $\alpha \ll 1$ to this case.

As with prior work on both private and robust mean estimation (e.g. \cite{cherapanamjeri2019fast,hopkins2021efficient}) our algorithm produces a series of iterates $x_0,x_1,\ldots,x_T$ for $T = O(\log d)$, where $x_0$ is the origin, with the invariant that $\|x_t - \mu\| \leq 0.9 \|x_{t-1} - \mu\|$.
To accomplish this, given $x_{t-1}$ with $\|\mu - x_{t-1}\| \gg 1$, and samples $X_1,\ldots,X_n \sim \cN(\mu,I)$, we find a unit vector $v$ such that $\iprod{v,\mu - x_{t-1} } \geq 0.9 \|\mu - x_{t-1}\|$; then we could take $x_t = x_{t-1} + \Omega(\|\mu - x_{t-1}\|) v$.

Unlike prior works, for reasons we will see shortly, we also need the invariant
that $x_t$ is $k$-sparse.  So, we actually take $x_t$ to be $x_{t-1} + \Omega(\|\mu - x_{t-1}\|) v$
with all but the largest-magnitude $k$ coordinates set to $0$; we show that this thresholding
step cannot increase the distance to $\mu$ by too
much (\cref{lem:sparsify-iterates}).

\textbf{Picking a gradient in exponential time: }
To find such a gradient vector $v$, we use the SoS exponential mechanism \cite{hopkins2021efficient}.
Let us first see how we would pick a gradient vector $v$ using the (non-SoS) exponential mechanism, but allowing exponential running time.
For a given dataset $\cX = X_1,\ldots,X_n$, the goal is to find a \emph{score function} $s_{\cX}(v) \in \R$ which assigns each $(2k)$-\emph{sparse} unit vector a score, such that for neighboring datasets $\cX,\cX'$ we always have $|s_{\cX}(v) - s_{\cX'}(v)| \leq 1$ (``bounded differences'').
Here we choose $2k$-sparse because $\mu-x_{t-1}$ is itself $2k$ sparse.
Then outputting a random $v$, where each is chosen with probability $\propto \exp(\e s_{\cX}(v))$, gives an $\e$-DP mechanism.
The goal is that $v$s with high scores are closer to $\mu - x_{t-1}$.

A good choice of score function turns out to be $s_{\cX}(v) = \sum_{i \leq n} \mathbf{1}\{\iprod{v, X_i - x_{t-1}} \geq \|x_{t-1} - \mu\| - O(1)\}$ -- using standard concentration of measure one can show that as long as $n \gg k \log d$, for $v$ s.t. $s_{\cX}(v) \geq 0.9n$ we will have $\iprod{v,\mu-x_{t-1}} \geq 0.9 \|\mu-x_{t-1} \|$ (``utility'').
Furthermore, if one were to sample a uniformly random $k$-sparse unit vector, it would satisfy $\iprod{v, \mu - x_{t-1}} \geq 0.9 \|\mu - x_{t-1}\|$ with probability $d^{-O(k)}$ (since $\mu - x_{t-1}$ is $2k$-sparse).
So, the distribution given by $\Pr(v) \propto \exp(\e s_{\cX}(v))$ puts $1-\beta$ probability on $v$ with $s_{\cX}(v) \geq 0.9n$, so long as $n \gg (k \log d + \log(1/\beta))/\e$, since it ``boosts'' the probabilities of these high scoring vectors by a factor of $\exp(\Omega(d \log k)) \geq d^{\Omega(k)}$.

Here we have crucially used the fact that $x_{t-1}$, and hence $\mu - x_{t-1}$, is sparse: otherwise, the gradient $v$ we need to select would not be sparse, and we would need to use exponential mechanism to sample $v$ from a bigger set.
This, in turn, would require us to draw more samples $n$, to ensure that the score function is well-behaved for a bigger set of vectors, because the probability of $\iprod{v,\mu-x_{t-1}} \geq 0.9 \|\mu -x_{t-1}\|$ for uniformly-random $v$ would be $\ll d^{-k}$.

Of course, the major drawback of the above is that it is not clear how to sample from the necessary distribution of $v$s efficiently -- in fact, doing so would violate the lower bound of \cref{thm:sparse-lb}.
There is also a second drawback: to evaluate the score function given above, we would need to know $\|\mu - x_{t-1}\|$; however, we are able to adapt the strategy of \cite{hopkins2021efficient} for this task to the sparse setting, re-using several of the ideas below.

\textbf{From exponential to polynomial time with SoS exponential mechanism: }
The SoS exponential mechanism allows potentially exponential-time instances of the exponential mechanism to be converted into polynomial time algorithms, so long as (a) the bounded-differences and utility properties of the score function can be proved in a certain restricted proof system (the \emph{SoS proof system}), and (b) the set over which the exponential mechanism is run is convex.

\emph{Convexity: }
The $2k$-sparse unit vectors -- used by the above exponential-time algorithm -- do not form a convex set.
A natural idea is to relax from the $2k$-sparse vectors to the (scaled) $\ell_1$ ball.
This creates a substantial difficulty: the set $\left\{ v \, : \, \|v\|_1 \leq O\left(\sqrt{k}\right) \right\}$ has much more volume near the origin than the set of $2k$-sparse unit vectors.
In particular, it is no longer true that $\iprod{v,\mu-x_{t-1}} \geq 0.9 \|\mu -x_{t-1}\|$ with probability $d^{-O(k)}$ for uniformly-random $v$; this probability will be exponentially small in $d$.

To fix this, we ``fatten'' the $\ell_1$ ball: we use the exponential mechanism
over the set $\cC = \left\{ v \, : \exists w \text{ s.t. } \|v-w\| \leq 0.01, \|w\|
\leq 1, \|w\|_1 \leq O\left(\sqrt{k}\right) \right\}$.
While perhaps counterintuitive that \emph{adding} vectors is helpful here, it is
possible (\cref{lem:l1-fattening})
to show roughly the following statement: over randomly-chosen $v$ in this set, $\Pr(\iprod{v,\mu - x_{t-1}} \geq 0.9 \|\mu - x_{t-1}\|) \geq e^{-O(d/\sqrt{k})} \cdot d^{-O(k^2)} = d^{-O(k^2)}$ for $k \geq d^{0.4}$.
To see this, note that all $w$ with $\|w - \tfrac{\mu - x_{t-1}}{\mu - x_{t-1}}\| \leq 0.01$ are in $\cC$ and have $\iprod{w, \mu-x_{t-1}} \geq 0.0 \|\mu - x_{t-1}\|$, and, using Sudakov minoration, $\cC$ can be covered by $d^{O(k^2)}$ $\ell_2$-balls of radius $1/\sqrt{k} + 0.01$.
So at least a $d^{-O\left(k^2\right)} \left( \tfrac {0.01} {0.01 + 1/\sqrt{k}} \right)^d = d^{-O(k^2)} \exp(-O(d/\sqrt{k}))$-fraction of the volume of $\cC$ lies on such ``good'' $w$s.
So, as long as $n \gg (k^2 \log d + \log(1/\beta))/\e$ and $k \geq d^{0.4}$, the arguments we used for the exponential-time method will still work.

\emph{Utility and Bounded-Differences in SoS: }
Finally, we turn to (a), the need to capture the proofs of the bounded-differences and utility in the SoS proof system.
Here we take ``intuitively simple'' as a proxy for ``expressible in SoS,'' deferring technical definitions to \cref{sec:strongly-private-sparse}.
The bounded-differences property for sum-of-threshold based score functions like $s_{\cX}$ above has a standard SoS proof \cite{hopkins2020mean}, so we won't address it further here.

The concentration of measure arguments we referred to above to establish utility are not captured by the SoS proof system!
So we actually need to devise a new proof of utility for the score function $s_{\cX}$.
The key step is a proof that if $n \gg (\log d)^{O(1)}(k^2 + \log(1/\beta))$ then with probability $1-\beta$, for every unit vector $w$ with $\|w\|_1 \leq \sqrt{k}$,  we have $\sum_{i \leq n} \mathbf{1}(\iprod{X_i - \mu,w} \gg 1) \leq 0.1n$.
To see that this implies utility, note that if $v$ has $s_{\cX}(v) \geq 0.9n$, then there exists some $i$ such that $\iprod{v, X_i - x_{t-1}} \geq \|\mu - x_{t-1}\| - O(1)$ but $\iprod{v, X_i - \mu} \leq O(1)$.
Then $\iprod{v,\mu - x_{t-1}} = \iprod{v,X_i - x_{t-1}} + \iprod{v,\mu-X_i} \geq \|\mu - x_{t-1}\| - O(1)$.

By a standard bounded-differences argument, it suffices to prove that $\E \max_w \sum_{i \leq n} \textbf{1}(\iprod{X_i - \mu,w} \gg 1) \leq 0.01n$.
Letting $M = \sum_{i \leq n} (X_i - \mu)(X_i - \mu)^\top$, by Cauchy-Schwarz, the quantity on the left-hand side is at most $0.01 \sqrt{n} \cdot (\E \max_w \iprod{M, ww^\top})^{1/2}$.
Splitting $M = M_{\text{diag}} + M_{\text{off-diag}}$, and applying Holder's inequality, $\iprod{M,ww^\top} \leq \|M_{\text{diag}}\|_\infty \|w\|_2^2 + \|M_{\text{off-diag}}\|_\infty \|w\|_1^2$, where $\|\cdot\|_\infty$ is the entry-wise $\ell_\infty$ norm.
The important term here turns out to be the second one; it is possible to bound $\E \|M_{\text{off-diag}}\|_\infty \leq (\log d)^{O(1)} \sqrt{n}$, for an overall bound of $0.01 n^{3/4} k$ (up to log factors), which gives the desired bound if $n \gg k^2$.
See \cref{lem:covariance-subset} and \ref{lem:sos-utility}
for the formal version of this argument.

\paragraph{Fast algorithms.} Finally,
for~\cref{alg:k2_logR_algo_peeling}~(\texttt{Threshold} in
\cref{table:algorithm-results}), we use standard DP tools to guarantee privacy
and utility --- the exponential mechanism~\cite{mcsherry2007mechanism}
with a sensitivity-$1$ score function in a coordinate-wise fashion, followed by
a black-box application of a univariate mean estimator.
Formal arguments are given in \cref{sec:experimental-details}.

%% file: sections/sparse_inf_theory.tex
\section{Information-theoretic Results for Private Sparse
         Mean Estimation}\label{sec:sparse-inf-theory}

In this section we focus on algorithms and lower bounds for private sparse mean
estimation in the absence of computational considerations. First, we consider
the Gaussian case; then we develop algorithms for the more general heavy-tailed
case where we only assume that the data has bounded covariance.
\subsection{Gaussian Private Sparse Mean Estimation}

\begin{proposition}\label{prop:sparse-exp-gaussian}
  For every $\alpha, \e, \beta, R > 0$ and small-enough
  $\eta > 0$ there is an $\e$-DP, mechanism taking  $ n \gg
  \frac{k\log(d)}{\alpha^2} + \frac{\log(1/\beta) +
  k\log\left(Rk/\alpha\right)}{\alpha\epsilon} $ $\eta$-corrupted samples
  from $\cN(\mu, I)$ with $\mu \in \R^d$, $\|\mu\|_0 \leq k$, and $\|\mu\| \leq
  R$ and produces $\hat{\mu}$ such that $\|\mu - \hat{\mu}\| \leq \alpha +
  O(\eta)$, with probability $1-\beta$.
\end{proposition}
\cref{prop:sparse-exp-gaussian} follows quickly from combining the private
hypothesis selection mechanism of \cite{bun2019private} with \cref{thm:meta} to
establish robustness. For convenience, we restate the result
of~\cite{bun2019private}. Before that, we introduce the notion of a Scheff\'e set.

\begin{definition}[Scheff\'e set]
    Let $\mathcal{H}$ be a set of distributions on $\mathcal{X}$.
     The Scheff\'e set for $H, H'\in\mathcal{H}$ is defined as
     $w(H, H') = \left\{x\in\mathcal{X} \mid H(x) > H'(x)\right\}$. The set $\mathcal{W}$
     of Scheff\'e sets for $\mathcal{H}$ is given by $\mathcal{W}(\mathcal{H}) =
     \left\{w(H, H') \mid H, H' \in \mathcal{H}\right\}$.
\end{definition}

\begin{theorem}[{{\cite[Theorem 4.1]{bun2019private}}}]\label{thm:bun_hypothesis_selection_from_VC}
Suppose $\mathcal{H}$ is a set of distributions on $\mathcal{X}$. Let $d$ be the
VC dimension of the set of indicators of the Scheff\'e sets of $\mathcal{H}$.
Then there exists an $(\epsilon, \delta)$-differentially private mechanism with
the following guarantees. Suppose $D=\{X_1, \ldots, X_n\}$ is a set of private
samples independently drawn from an unknown distribution $P$ and suppose there
exists some $H^*$ such that $\dtv(P, H^*)\leq \alpha$. If 
$n = \Omega\left( \frac{d + \log(1/\beta)}{\alpha^2} +
                  \frac{\log(|\mathcal{H}|/\beta) +
                  \min\{\log(|\mathcal{H}|), \log(1/\delta)\}}
                  {\alpha\epsilon} \right)$, then the output
$\hat H$ of the algorithm is such that $\dtv(P,\hat H)\leq 7\alpha$ with
probability at least $1-\beta$.
\end{theorem}

Intuitively,~\cref{thm:bun_hypothesis_selection_from_VC} tells us that if we
have a hypothesis class of candidate distributions with a small cover, an
example of which are Gaussians with bounded mean and known covariance, then we
can privately select a distribution close to an unknown target distribution in
total variation distance with only a mild additional cost of privacy. To make
this more precise, we need the following definition.

\begin{definition}
A set $\acov$ of distributions is an $\alpha$-cover for a set of
distributions $\mathcal{H}$ if for every $H\in\mathcal{H}$, there is some
$C\in\acov$ such that $\dtv(H, C)\leq \alpha$.
\end{definition}

In order to apply~\cref{thm:bun_hypothesis_selection_from_VC} to our use case, we
first show a cover for Gaussian distribution with $k$-sparse bounded means.
\begin{lemma}\label{lem:cover_gaussians_with_sparse_mean}
Let $\mathcal{S}$ be the set of
Gaussian distributions $\cN(\mu, I)$ in $d$ dimensions with $k$-sparse
$\mu$ such that $\|\mu\|_2\leq R$. Then $\mathcal{S}$ admits an $\alpha$-cover
of size
\begin{equation*}
    O\left(
        \binom{d}{k}\left(\frac{Rk}{\alpha}\right)^k
     \right).
\end{equation*}
\end{lemma}
\begin{proof}
First, observe that it suffices to obtain an $\alpha$-cover for each of the
$\binom{d}{k}$ possible choices of the support of $\mu$, and return the union of
the covers. For any particular subset, the problem becomes equivalent to finding
an $\alpha$-cover of the set of non-sparse Gaussian distributions with bounded
mean in $k$ dimensions. By Lemma 6.8 in~\cite{bun2019private}, there exists
an $\alpha$-cover of size $O\left(\frac{Rk}{\alpha}\right)$. For completeness,
we note that the proof of Lemma 6.8 in~\cite{bun2019private} constructs the
$\alpha$-cover by taking the Cartesian product of $(\alpha/k)$-covers in each
standard basis direction, which by the triangle inequality gives an
$\alpha$-cover.
\end{proof}

Covering numbers are tightly related to VC dimension. The final piece we need
for~\cref{prop:sparse-exp-gaussian} is a bound on the VC dimension of the
Scheff\'e sets of the set of $k$-sparse $d$-dimensional Gaussians with identity
covariance. 
\begin{lemma}\label{lem:VC-dim-sparse-gaussians}
   Let $\mathcal{H}_k^d$ be the set of Gaussian distributions with covariance matrix
   $I_d$ and mean $\mu\in\R^d$ such that $\|\mu\|_0 \leq k$. Let
   $\mathcal{W}\left(\mathcal{H}_k^d\right)$ be the set of indicators of
   Scheff\'e sets of $\mathcal{H}_k^d$. Then $\mathcal{W}$ has VC dimension at
   most $4k\log(de) = O(k\log(d))$.
\end{lemma}
\begin{proof}
    We want to bound the VC dimension of the set of functions $f_{H,
    H'}\colon \mathcal{X}\to \{0,1\}$ defined by $f_{H, H'}(x) \mapsto
    \1\{H(x) > H'(x)\}$  for $H, H'\in \mathcal{H}_k^d$.

    First, observe that for any two Gaussians $H, H'$, $f_{H,H'}$ correspond to
    a linear threshold function $L(x) = \left\{ a\in \R^d \mid \langle a,
    x\rangle \geq 0 \right\}$. Restricting $H, H'$ to belong to
    $\mathcal{H}_k^d$, i.e. have $k$-sparse means, we now have that each $f_{H,
    H'}$ corresponds to a linear threshold function $L(x)$ generated by
    $2k$-sparse vector $x$. We now leverage a result of Ahsen and
    Vidyasagar~\cite[Theorem 6]{ahsen2019approach} which states that the set of
    linear threshold functions generated by $k$-sparse vectors in $d$ dimensions
    has VC dimension at most $2k\log(de)$; this concludes the argument.
\end{proof}

We are now ready to prove~\cref{prop:sparse-exp-gaussian}.
\begin{proof}[Proof of~\cref{prop:sparse-exp-gaussian}]
    Consider the case of uncorrupted data ($\eta=0$). We construct an estimator
    that gives a guarantee in $\ell_2$ distance, which we show is equivalent to
    total variation distance for spherical Gaussians.  

    From~\cref{lem:cover_gaussians_with_sparse_mean} and a bound on the binomial
   coefficients we obtain a cover of size
    \begin{equation}
        O\left( (de/k)^k\left(\frac{Rk}{\alpha}\right)^k \right).
    \end{equation}
    From~\cref{lem:VC-dim-sparse-gaussians}, we know that the VC dimension of
    the set of indicators of the Scheff\'e sets of $k$-sparse Gaussians with
    identity covariance in $d$ dimensions is $O(k\log(d))$.

    With the size of the $\alpha$-cover and the bound on the VC dimension established,
    we are ready to invoke~\cref{thm:bun_hypothesis_selection_from_VC}. This
    gives us the guarantee that for the output distribution $\hat H$ we have
    $\dtv(P, \hat H)\leq 7\alpha$ with probability at least $1-\beta$. Suppose
    the selected distribution $\hat H$ has mean $\hat\mu$.  Since both $P$ and
    $\hat H$ are spherical Gaussians, we have that
    \begin{equation}
    \|\mu-\hat\mu\|_2 \leq \dtv(P, \hat H)
    \end{equation}
    following from~\cite[Theorem~1.2]{devroye2018total}. Therefore, we can
    obtain an estimate with the desired properties by outputting the mean of the
    selected distribution $\hat H$.

    Finally, robustness follows directly from~\cref{thm:meta}.
\end{proof}

As for an information-theoretic lower bound,~\cite{cai2021cost} give a
tracing-based lower bound (which also holds for approximate DP) which states
that any $\epsilon$-DP algorithm must use at least
$n=\Omega\left(\frac{k\log(d)}{\alpha^2} +
\frac{k\log(d)}{\alpha\epsilon}\right)$
samples to estimate the mean of a sub-Gaussian distribution up to $\ell_2$ error
$\alpha$, assuming $\|\mu\|_\infty\leq 1$. Alternatively, we can leverage
hypothesis-selection based arguments again, to get a
$\Omega\left(\frac{k\log(d) + \log (Rk/\alpha)}{\e}\right)$ bound, assuming
$\|\mu\|_2\leq R$.

\subsection{Heavy-Tailed Private Sparse Mean Estimation}
Now we turn to a more general set of distributions -- ones with bounded
covariance. We give an (exponential-time) algorithm --~\cref{alg:bounded_cov},
which has strong group privacy properties allowing us to invoke~\cref{thm:meta}.

\begin{algorithm}[h!]
    \caption{Exponential-time algorithm for private sparse heavy-tailed mean
    estimation}\label{alg:bounded_cov}
    \hspace*{\algorithmicindent}
    \textbf{Input:} bucket size $b$, coordinate-wise variance
    $\sigma^2$, number of samples $n$\\
    \begin{algorithmic}[1]
    \Statex \texttt{// Support estimation}
    \State $m_j \gets \frac{1}{b}\sum_{j=i\cdot b}^{(i+1)b} x_i$ for all
    $j\in[\lceil n/b\rceil]$
    \Comment{compute bucketed means}

    \State $s_T \gets S(\{m_i\}_{i\leq \lceil n/b\rceil}, T, L)$ for all
    $T\subseteq[d]$ with $|T|=k$
    \Comment{compute scores}

    \State Run the exponential mechanism~\cite{mcsherry2007mechanism} on
    $\{s_T\}_{T:T\subseteq[d], |T|=k}$ to get candidate support $T$

    \Statex \texttt{// Coarse dense mean estimation (in ambient dimension $k$)}
    \State
    $S_{coarse}(\{x_{iT}\}_{i\leq n}, x, L) = \sum_{i\leq n}
    \1\{\|x_{iT} - x\| \geq L\}$
    \State $s_x \gets S_{coarse}(\{x_{iT}\}_{i\leq n}, x, L)$ for all $x$ in a
    $\sqrt{k}$-packing of the $\ell_2$ ball of radius $R$
    \State Run exponential mechanism on $\{s_x\}$ to get $\hat\mu_{coarse}$

    \Statex \texttt{// Fine dense mean estimation (in ambient dimension $k$)}
    \State
    $S_{fine}(\{x_{iT}\}_{i\leq n}, x, L; \hat\mu_{coarse}) =
    \max_{v:\|v\|=1}\sum_{i\leq n}\1\{\iprod{x_{iT} - x,
    v} \geq L\}$
    \State $s_x \gets S_{fine}(\{m_{iT}\}_{i\leq \lceil n/b\rceil}, x, L)$ for
    all $x$ in a $1$-packing of the $\ell_2$ ball of radius $\sqrt{k}$ centered
    at $\hat\mu_{coarse}$
    \State Run exponential mechanism on $\{s_x\}$ to get $\hat\mu$
    \State $\hat\mu_i\gets
    \begin{cases}
        \left(\hat\mu_T\right)_i\ \text{if } i\in T\\
        0\ \ \ \ \ \ \ \ \ \text{o.w.}
    \end{cases}$
    \end{algorithmic}
    \textbf{Output:} mean estimate $\hat\mu$
\end{algorithm}

\begin{theorem}\label{thm:sparse-exp-bcov}
For every $\alpha, \e, \beta, R > 0$ and small-enough $\eta >
0$~\cref{alg:bounded_cov} is an $\e$-DP mechanism with the following properties.
Given $n$ $\eta$-corrupted samples from any distribution $D$ with mean $\mu \in
\R^d$ having $\|\mu\|_0 \leq k, \|\mu\| \leq R$, and covariance $\Sigma \preceq
I$, it produces an estimate $\hat{\mu}$ such that $\|\mu - \hat{\mu}\| \leq
\alpha + O(\sqrt{\eta})$, with probability $1-\beta$, as long as
    \[n = \Omega\left(
   \underbrace{\frac{k\log(d) + \log(1/\beta)}{\alpha^2\epsilon}}_{\text{support estimation}} +
   \underbrace{
   \frac{k\log(k) + \log(1/\beta)}{\alpha^2\e} + \frac{k\log(R)}{\e}
   }_{\text{dense estimation in }k\text{ dimensions}}
    \right).\]
\end{theorem}

To prove \cref{thm:sparse-exp-bcov} we can no longer rely on the hypothesis
selection mechanism of \cite{bun2019private}, as the hypothesis class of
distributions with bounded second moments doesn't have a finite cover in total
variation distance. 
Instead, we design a new exponential-mechanism-based approach to identify the
nonzero coordinates of $\mu$, inspired by quantile-based score functions used in
robust statistics.

We show that the exponential mechanism is going to select a subset
$\hat T$ of coordinates such that $\|\mu_{\hat T}\|$ is close to $\|\mu\|$ with
high probability. We are guaranteed that the exponential mechanism will select
a candidate with score
\begin{equation}
    \label{eq:exp_mech_guarantee}
   \opt - 2\frac{\Delta^{(S)}}{\epsilon}\log\left(|\mathcal{H}|/\beta)\right)
\end{equation}
with probability at least $\beta$. We first show that the sensitivity of the
score function $S$ is $1$, which directly affects the guarantee
in~\cref{eq:exp_mech_guarantee}.

For convenience, we restate the definition of $S$ below:
\begin{equation}
    S(\{x_i\}_{i\leq n}, T; L) = \max_{v\in\mathbb{R}^k, \|v\|_2=1}
                    \sum_{i=1}^n \1\{ v^\top x_{iT} \geq L\}.
\end{equation}
To ease notation, we will omit the $\{x_i\}_{i\leq n}$ argument whenever it is
clear from context.

\begin{lemma} \label{lem:sensitivity_score_sparse_mean}
    The score function $S(\cdot, \cdot; L)$ has sensitivity $\Delta^{(S)} = 1$
    for all values of $L$.
\end{lemma}
\begin{proof}
    Changing a single element $x_i$ in the dataset affects at most one of the
    indicators in the sum, thus changing $S$ by at most $1$.
\end{proof}

Next, we show that the ground-truth subset has a high score.
\begin{lemma}
    \label{lem:score_of_opt}
    Let $x_1, \ldots, x_n\sim P$ be independent samples and suppose that $P$ has
    a $k$-sparse mean $\mu$ such that $\|\mu\|_2\leq R$, and
    $\|\Sigma\|_{op}\leq \lambda$.  Then for the ground-truth subset $T^*$ and
    some $c > 0$ we have $S(\{x_i\}_{i\leq n}, T^*; \|\mu\| - c)\geq n -
    \frac{n\lambda}{c^2} - \sqrt{n\log(3/\beta)/2}$ with probability
    at least $1-\beta/3$. 
\end{lemma}
\begin{proof}
    Since $T^*$ is the optimal subset, we have that $\|\mu_{T^*}\|=\|\mu\|$.
    Using $v=\frac{\hmut}{\|\hmut\|}$ as a representative, we get
    \begin{equation}
        \begin{split}
            \E S(T^*;L) &\geq \E \sum_{i=1}^n \1\{ \hmut^\top \xit \geq \|\hmut\| L \} \\
            &= \sum_{i=1}^n \P\left(  \hmut^\top \xit \geq \|\mu\| L  \right).
        \end{split}
    \end{equation}
    Now we bound $\var(\hmut^\top \xit)$. Let $\Sigma_{T^*}$ be the
    covariance matrix
    restricted to the coordinates in $T^*$.
    By Cauchy's interlacing theorem, we have that $\|\Sigma_{T^*}\|_{op}\leq
    \lambda$. Thus, 
    \begin{equation}
        \var(\hmut^\top \xit) \leq \|\hmut\|^2\lambda = \|\mu\|^2\lambda.
    \end{equation}
    This allows us to use Chebyshev's inequality on each term of the sum to get
    \begin{equation}
        \begin{split}
            \P\left( \hmut^\top \xit \geq \|\mu\| L \right)
            &= 1 - \P\left( \hmut^\top \xit - \|\mu\|^2 < \|\mu\|L - \|\mu\|^2 \right) \\
            &\geq 1 - \frac{\|\mu\|^2\lambda}{\left(\|\mu\|L-\|\mu\|^2\right)^2}\\
            &= 1 - \frac{\lambda}{(L - \|\mu\|)^2} = 1 - \frac{\lambda}{c^2}
        \end{split}
    \end{equation}
    as long as $L < \|\mu\|$, which is guaranteed from the condition that $c > 0$.
    Thus, in total, we have
    \begin{equation}
        \E S(T^*; \|\mu\| - c) \geq n - \frac{n\lambda}{c^2}.
    \end{equation}
    Finally, from McDiarmid's inequality, we have that
   \begin{equation}
       \P\left(S(T;L) - \E S(T;L) \geq t \right) \leq \exp\left(-\frac{2t^2}{n}\right),
   \end{equation}
   which concludes the proof.
\end{proof}

In addition to the ground-truth subset of coordinates having a large score, we
want ``bad'' subsets to have a low score.

\begin{lemma}
    \label{lem:score_of_bad_subsets}
Let $x_1, \ldots, x_n\sim P$ be independent samples and suppose that $P$ has a
$k$-sparse mean $\mu$ such that $\|\mu\|_2\leq R$, and $\|\Sigma\|_{op}\leq
\lambda$. Let $ T\subseteq [d]$. Then for any
$L$ such that $\|\mu_{T}\| < L$ we have that
$S(T; L) \leq n\frac{\sqrt{\lambda}}{L-\|\mu_T\|}+ \frac{2}{L}\sqrt{n k
\lambda} + \sqrt{n\log(3/\beta)}$ with probability at least $1-\beta / 3$.
\end{lemma}
\begin{proof}
   First, note that $S$ has bounded differences with respect to the sequence of
   random vectors $X=\{X_1,\ldots, X_n\}$. In particular,
   from~\cref{lem:sensitivity_score_sparse_mean} we have that
   $|S(X;L)-S(X';L)|\leq 1$ for any $X, X'$ that differ in one entry. Thus, we
   can apply McDiarmid's inequality to obtain
   \begin{equation}
       \P\left(S(T;L) - \E S(T;L) \geq t \right) \leq \exp\left(-\frac{2t^2}{n}\right)
   \end{equation}
   for all $L, T$, and $t$. The rest of the proof is devoted to bounding $\E S( T;
   L)$. To ease notation, for $X=\{x_1,\dots, x_n\}$ define
   \begin{equation}
       f_{T, L, v}(X) = \sum_{i=1}^n\1\{v^\top x_{iT} \leq L\}.
   \end{equation}
   We start with
   \begin{equation}
       \begin{split}
           &\E S(T;L) = \E_X \maxv f_{T, L, v}(X) \\
           &\leq \E_X\left[ \maxv\left( f_{T, L, v}(X) - \E f_{T, L, v}(X) \right)
                            + \maxv\E f_{T, L, v}(X) \right] \\
           &= \E_X\maxv\left( f_{T, L, v}(X) - \E f_{T, L, v}(X) \right)
                            + \maxv\E f_{T, L, v}(X) \\
           &= \E_X\maxv\left( f_{T, L, v}(X) - \E f_{T, L, v}(X) \right)
                            + n\maxv\P_{X}\left(v^\top x_{T}\leq L\right).
       \end{split}
   \end{equation}
   We follow with a symmetrization argument for the first term. Let
   $\{\sigma_i\}_{i=1}^n$ be independent Rademacher random variables. Then
   \begin{equation}
       \begin{split}
           &\E_X\maxv\left( f_{T, L, v}(X) - \E f_{T, L, v}(X) \right) \\
           &\leq \E_{X, X',\sigma_i}\maxv\left(\sigma_i(f_{T, L, v}(X) - f_{T, L, v}(X')) \right) \\
           &\leq \E_{X, X',\sigma_i}\maxv\left(
                \sum_{i\leq n} \sigma_i \left(\1\{v^\top x_{iT} \leq L\} - \1\{v^\top x'_{iT} \leq L\} \right)
           \right) \\
           &\leq 2\E_{X, \sigma_i}\maxv\left( \sum_{i\leq n} \sigma_i \1\{v^\top x_{iT} \leq L\} \right),
       \end{split}
   \end{equation}
   via the triangle inequality.  Now we can use the inequality
       $\1\left\{\eta^\top v \leq L\right\} \leq \frac{|\eta^\top v|}{L}$
   and Ledoux-Talagrand contraction~\cite{lugosi2019sub} to get
    \begin{equation}
        \begin{split}
           & 2\E_{X, \sigma_i}\maxv\left( \sum_{i\leq n} \sigma_i \1\{v^\top x_{iT} \leq L\} \right)\\
           &\leq 2\E_{X, \sigma_i}\maxv \sum_{i\leq n} \sigma_i \frac{|v^\top x_{iT}|}{L} \\
           &\leq \frac{2}{L}\E_{X, \sigma_i}\maxv \sum_{i\leq n} \sigma_i \cdot v^\top x_{iT} \\
           &= \frac{2}{L}\E_{X, \sigma_i} \maxv v^\top\left(\sum_{i\leq n} \sigma_i \cdot x_{iT} \right) \\
           &= \frac{2}{L}\E_{X, \sigma_i} \left\|\sum_{i\leq n} \sigma_i \cdot x_{iT} \right\| \\
           &\leq \frac{2}{L} \sqrt{n\Tr(\Sigma_T)},
        \end{split}
    \end{equation}
    where $\Sigma_T$ is the restriction of the covariance matrix $\Sigma$ to the
    coordinates of the subset $T\subseteq [d]$. By the Cauchy interlacing
    theorem we have that $\|\Sigma_T\|_{op}\leq \|\Sigma_T\|$ and thus overall
    we have
    \begin{equation}
        \E_X\maxv\left( f_{T, L, v}(X) - \E f_{T, L, v}(X) \right) \leq \frac{2}{L} \sqrt{nk\lambda}.
    \end{equation}

   For the second term we have
   \begin{equation}
    \begin{split}
       & n \maxv \P_X (v^\top x_{T}\leq L)\\
       &= n \maxv \P_X (v^\top x_{T} -  \|\mu_T\| \leq L - \|\mu_T\|)\\
       &\leq n \maxv \P_X (v^\top x_{T} - v^\top \mu_T \leq L - \|\mu_T\|)\\
       &= n \maxv \E_X \left[\1\left\{  (x_{T} -  \mu_T)^\top v \leq L - \|\mu_T\|\right\}\right]\\
       &\leq n\maxv \E_X\left[ \frac{|(x_T-\mu_T)^\top v|}{L-\|\mu_T\|} \right]
       \leq \frac{n\|\Sigma\|^{1/2}}{L-\|\mu_T\|}.
    \end{split}
   \end{equation}
\end{proof}

Finally, we show that the dense (fine and coarse) private mean estimation
algorithms satisfy the properties desired by~\cref{alg:bounded_cov} in
Lemmas~\ref{lem:bounded-cov-coarse-dense} and~\ref{lem:bounded-cov-fine-dense}.

\begin{lemma}[Coarse dense estimation]\label{lem:bounded-cov-coarse-dense}
    Let $x_1,\ldots,x_n \sim P$ be independent samples and suppose $P$ has
    covariance $\Sigma \preceq I$ and mean $\mu\in\R^k$ such that $\|\mu\|_2\leq
    R$. Then running exponential mechanism on $\{x_i\}_{i\leq n}$  with the
    score function  $S_{coarse}(\{x_{i}\}_{i\leq n}, x, L) = \sum_{i\leq n}
    \1\{\|x_{i} - x\| \geq L\}$ gives us an estimate $\hat\mu$ such that
    $\|\hat\mu - \mu\|_2 \leq \sqrt{k}$ with probability at least $1-\beta$ so
    long as $n =\Omega\left(\frac{k\log R + \log(1/\beta)}{\e}\right)$.
\end{lemma} 

\begin{lemma}[Fine dense estimation]\label{lem:bounded-cov-fine-dense}
    Let $x_1,\ldots,x_n \sim P$ be independent samples and suppose $P$ has
    covariance $\Sigma \preceq I$ and mean $\mu\in\R^k$ such that $\|\mu\|_2\leq
    1$ and $\|\mu\|_0\leq k$. Then running exponential mechanism on
    $\{x_i\}_{i\leq n}$ with the score function $S_{fine}(\{x_{i}\}_{i\leq n},
    x, L) = \max_{v:\|v\|=1}\sum_{i\leq n}\1\{\iprod{x_{i} - x, v} \geq L\}$
    gives us an estimate $\hat\mu$ such that $\|\hat\mu - \mu\|_2 \leq \alpha$
    with probability at least $1-\beta$ so long as $n =\Omega\left(\frac{k\log k
    + \log(1/\beta)}{\alpha^2\e}\right)$.
\end{lemma} 

For both the coarse and fine estimation, our proof strategy will be as follows:
first, argue that the sizes $|P_{coarse}|$ and $|P_{fine}|$ of the packings are
small enough such that $\max\{\log(|P_{coarse}|), \log(|P_{fine}|)\}\leq k\log
k/\alpha^2 + k\log(R)$. After that, we will show that low score (few
``outliers'') is achieved if and only if we are at a good mean candidate
($\alpha$-close to $\mu$).  

\begin{proof}[Proof of \cref{lem:bounded-cov-coarse-dense}]
    The size of $\sqrt{k}$-packing of the $\ell_2$ ball of radius $R$ in $k$ dimensions
    is $O\left({R}^k\right)$. Thus, $\log(|P_{coarse}|) \leq k\log R$, as
    desired.

    The coarse estimation score function $S_{coarse}$ has sensitivity $1$, by an
    argument analogous to the one in~\cref{lem:sensitivity_score_sparse_mean}. 

    Now we turn to showing that $S_{coarse}$ is low if and only if we are at a
    mean candidate $\hat\mu$ such that $\|\hat\mu - \mu\|\leq \sqrt{k}$
    We can directly leverage standard concentration arguments to argue that any
    ball of $\mathrm{poly}(k)$ radius that contains at least $0.9 n$ of the
    samples must be at a distance at most $O\left(\sqrt{k}\right)$ to
    $\mu$~\cite{hopkins2020robust} with probability at least $1-\beta$. This is
    enough to show both directions of the desired implication. 

    Thus, in total, the coarse exponential mechanism gives us a coarse estimate
    with good utility with probability at least $1-\beta$ as long as
    $n=\Omega\left(\frac{\Delta}{\e}\left(k\log R + \log(1/\beta)\right)\right)
    =\Omega\left(\frac{k\log R + \log(1/\beta)}{\e}\right)$.
\end{proof}

\begin{proof}[Proof of \cref{lem:bounded-cov-fine-dense}]

    As in the proof of~\cref{lem:bounded-cov-coarse-dense}, the size of
    $\e$-packing (with $\ell_2$ balls) of the unit $\ell_2$ ball in $k$
    dimensions is $O\left(\left( \frac{1}{\e}\right)^k\right)$. Thus,
    $\log(|P_{fine}|) \leq k\log k/\alpha^2$, as desired.

    For the fine estimation part, without loss of generality assume that
    $\alpha=1$; bucketed means as employed in the support estimation part get us
    at the cost of an $\alpha^2$ factor in the sample complexity. We now show
    that a low score for a point $x$ implies that $\|\mu-x\|\leq 1$. More
    formally, let ``low'' mean $\leq 0.1n$. A low score implies that
    \begin{equation}
        \max_{v:\|v\|=1} \iprod{X_i - x, v}
        = \max_{v:\|v\|=1} \left(\iprod{X_i - \mu, v}  +  \iprod{\mu - x, v}\right)
        \leq L
    \end{equation}
    for at least $0.9n$ indices $i$. Taking the maximum for each term,
    this becomes
    $\max_{v:\|v\|=1} \iprod{X_i - x, v} \leq \max_{v:\|v\|=1}\iprod{X_i-\mu,v}
    + \|\mu-x\|$.
    From the bounded covariance assumption, we get that for at least $0.9n$ of
    the samples, we must have $\|\mu-x\|\leq L - 1$. Setting $L=2$ shows this
    direction. Suppose now we have a good mean candidate $x$. Then
    \begin{equation}
        \max_{v:\|v\|=1} \iprod{X_i - x, v}
        = \max_{v:\|v\|=1} \left(\iprod{X_i - \mu, v}  +  \iprod{\mu - x, v}\right)
        \leq \max_{v:\|v\|=1} \iprod{X_i - \mu, v} + 1.
    \end{equation}
    Now from standard concentration we have that
    $\max_{v:\|v\|=1}\iprod{X_i-\mu,v}\leq 1$ for at least $0.9n$ of the samples,
    and we are done.

\end{proof}

With this, we are ready to prove~\cref{thm:sparse-exp-bcov}.
\begin{proof}[Proof of \cref{thm:sparse-exp-bcov}]
    First, we note that bucketing with bucket size $b$ is done in order to
    reduce the variance of the samples. In particular, the bucketed means $m_i$
    have covariance matrix $\Sigma_m$ s.t. $\Sigma_m \preceq \lambda/b I$, where
    $\lambda := \lambda_{\max}(\Sigma)$. 

    From~\cref{eq:exp_mech_guarantee}, together with ~\cref{lem:score_of_opt}
    and~\cref{lem:score_of_bad_subsets}, we get that to avoid choosing a subset
    $T$ for which $\|\mu_T - \mu\| \geq \alpha$ with probability at most
    $\beta$, we need
    \begin{equation}
        n - \frac{n\lambda/b}{\alpha^2} - \sqrt{\frac{n\log(3/\beta)}{2}}
        - \frac{2(k\log(d) + \log(3/\beta))}{\epsilon}
        \geq n\frac{\sqrt{\lambda}}{2\sqrt{b}\alpha} + \frac{2}{L}\sqrt{\frac{nk\lambda}{b}} +
        \sqrt{\frac{n\log(3/\beta)}{2}},
    \end{equation}
    or equivalently
    \begin{equation}
        \sqrt{n}\left(1
        - \frac{\lambda}{b\alpha^2}
        - \frac{1}{2} \sqrt{\frac{{\lambda}}{{b}\alpha^2}}
        \right)
        \geq
        \frac{2(k\log(d) + \log(3/\beta))}{\sqrt{n}\epsilon}
        + \frac{2}{L}\sqrt{\frac{k\lambda}{b}}
        + 2\sqrt{\frac{\log(3/\beta)}{2}}.
    \end{equation}
    This holds true when $b \geq 25\frac{\lambda}{\alpha^2}$ and
    $n\geq
    \max\left\{
    10(k\log(d) + \log(3/\beta))/\epsilon,
    10k\lambda/\left(bL^2\right),
    10\log(3/\beta)
    \right\}$.
    This gets us the support estimation part of the sample complexity in the
    statement of~\cref{thm:sparse-exp-bcov}.

    Combining the support estimation step with the dense estimation steps from
    \cref{lem:bounded-cov-coarse-dense} and \cref{lem:bounded-cov-fine-dense}
    gives the desired result in the absence of corruptions. The $\alpha +
    \sqrt{\eta}$ rate for $\eta$-corrupted inputs follows directly
    from~\cref{thm:meta}. 
\end{proof}

It is worth noting that if we directly leverage existing heavy-tailed estimators
for the dense estimation
subroutine~\cite{kamath2020private,hopkins2021efficient}, we would get
suboptimal robustness guarantees -- either $\alpha +\sqrt{d\eta}$ or $\alpha +
\sqrt{\log(R)\eta}$ in the $\eta$-contamination model.

Now we turn to an information-theoretic lower bound for heavy-tailed sparse
private mean estimation.
\begin{proposition}\label{prop:sparse-lb-bcov}
   Suppose $M$ is an $\epsilon$-DP algorithm such that for every distribution
   $D$ with $k$-sparse mean $\mu$ with $\|\mu\|\leq R$ and covariance
   $\Sigma\preceq I$, $M$ produces an estimate $\hat\mu$ from $n$ samples such
   that $\|\mu - \hat{\mu}\| \leq \alpha$, with probability $1-\beta$. Then \[ n
   = \Omega\left( \frac{k\log d}{\alpha^2\epsilon}\right) .\]
\end{proposition}
\begin{proof}
    Our proof is a straightforward adaptation of the methods used to
    prove~\cite[Theorem 6.1]{kamath2020private}.
    In particular, we show the statement is true for a class of product
    distributions that satisfy the mean and covariance requirements.

    Formally, let
    \begin{equation} Q_0 =
        \begin{cases}
            -\frac{\sqrt{k}}{\alpha} \ w.p.\ \frac{\alpha^2}{k}\\
            0 \ \ \ \ \ \  w.p.\ 1 - 2\frac{\alpha^2}{k} \\
            \frac{\sqrt{k}}{\alpha} \ \ \ w.p.\ \frac{\alpha^2}{k}.
        \end{cases}
    \end{equation}
    and
    \begin{equation} Q_1 =
        \begin{cases}
            -\frac{\sqrt{k}}{\alpha} \ w.p.\ \frac{1}{2}\frac{\alpha^2}{k} \\
            0 \ \ \ \ \ \  w.p.\ 1 - 2\frac{\alpha^2}{k} \\
            \frac{\sqrt{k}}{\alpha} \ \ \ w.p.\ \frac{3}{2}\frac{\alpha^2}{k}.
        \end{cases}
    \end{equation}
    For $c\in\{0, 1\}^d$, define $Q_c = \bigotimes_{i=1}^d Q_{c_i}$.
    We will work only with $k$-sparse distributions, so we only need to consider
    $c$ such that $\|c\|_1=k$ Let $\cC_k$ denote the set of such $k$-sparse
    vectors in $\{0, 1\}^d$. For any $c, c'\in \cC_k$ we have that the maximum
    distance in total variation between $Q_{c}$ and $Q_{c'}$
    is at most $\alpha^2$. Now let $\cH$ be a linear code with Hamming distance
    $k/4$. This implies that for every $c, c'\in\cH$, we have that $c$ and $c'$
    differ on at least $k/4$ coordinates. From coding theory (Gilbert-Varshamov
    bound,~\cite[Lemma 3.1]{ba2010lower} for a more direct argument for
    $k$-sparse codes) we know that there exists a code of size
    $|\cH|=\Omega\left(d^{k}\right)$. The result follows from standard
    concentration and packing arguments.
\end{proof}

%% file: sections/strongly-private-alg.tex
\section{Polynomial-Time Sparse Mean Estimation with $\tilde{O}(k^2)$
Samples}\label{sec:strongly-private-sparse}
In this section we describe a polynomial-time $\e$-DP algorithm for sparse mean
estimation with strong group privacy guarantees -- it maintains privacy of
groups up to size $n/\poly \log(d)$. (Hence, the algorithm is automatically
robust to corruption of a $1/\poly \log(d)$ fraction of inputs, per
\cref{thm:meta}.)

\begin{algorithm}[h!]
    \caption{Polynomial-time algorithm for sparse sub-Gaussian mean
    estimation}\label{alg:k2_logR_algo_sos}
    \hspace*{\algorithmicindent}
    \textbf{Input:} iterations $N$, bucket size $b$,
    number of samples $n$, initial
    estimate $\mu_0$, step size $\eta$\\
    \begin{algorithmic}[1]
    \State $m_j \gets \frac{1}{b}\sum_{j=i\cdot b}^{(i+1)b} x_i$ for all
    $j\in[\lceil n/b\rceil]$
    \Comment{compute bucketed means}

    \For{$1\leq i\leq N$}
      \State Check \texttt{Halt-Estimation}$\left(\{m_j\}_{j=1}^n,
      \mu_{i-1}\right)$, return $\mu_{i-1}$ if \texttt{halt}
      \State $d_i\gets$\texttt{Distance-Estimation}$\left(\{m_j\}_{j=1}^n,
      \mu_{i-1}\right)$
      \State $g_i\gets$\texttt{Gradient-Estimation}$\left(\{m_j\}_{j=1}^n,
      \mu_{i-1}, d_i\right)$
      \State $\bar\mu_i\gets \mu_{i-1} + \eta d_i g_i$
      \State $\mu_i \gets$ \texttt{Sparsify}$(\bar\mu_i, k)$
    \EndFor

    \end{algorithmic}
    \textbf{Output:} mean estimate $\mu_T$
  \end{algorithm}

\begin{theorem}\label{thm:sos_sparse_mean_algo}
 There exist universal constants $C>0, C'>0$ such that for every $\e, R > 0$,
 $\beta \in (0,1)$, and $d,k \in \N$ such that $k\geq d^{0.4}$, there is a
 polynomial-time $\e$-DP algorithm with the following guarantees.  For every
 $k$-sparse $\mu \in \R^d$ with $\|\mu\| \leq R$, given independent samples
 $X_1,\ldots,X_n \sim \cN(\mu,I)$, with probability at least $1-\beta$ the
 algorithm outputs $x \in \R^d$ such that $\|\mu - x\| \leq \alpha$, so long as
 \[ n \geq C (\log d)^{C'} \log R \cdot \frac{ k^2 + \log(1/\beta) + \log \log R}{\alpha^2\e}
 \mper \]
\end{theorem}

To prove Theorem~\ref{thm:sos_sparse_mean_algo} we assemble our main lemmas here, and prove them in subsequent sections.
The first key lemma captures a private gradient-finding procedure, finding a direction correlated with $x - \mu$ for $k$-sparse $x$.

\begin{lemma}[Finding Gradients]
\label{lem:gradient}
For every $c > 0$ there exists $C > 0$ such that for every $\beta,R,\e > 0$
there is a polynomial-time $\e$-DP algorithm, \texttt{Gradient-Estimation}, with
the following guarantees. For every $d \in \N$ and $k \in \N$ with $k \geq
d^{0.4}$ and every $\mu \in \R^d$ with $\|\mu\| \leq R$, given 
\begin{itemize}
\item $n$ independent samples $X_1,\ldots,X_n \sim \cN(\mu,I)$,
\item a vector $x \in \R^d$ such that
\begin{itemize}
  \item $\|x\|_2 \leq R$,
  \item $\|x\|_0 \leq k$, and
  \item $\|x - \mu\| \geq C$
\end{itemize}
\item a number $r \in [(1-2/C) \|x - \mu\|, (1-1/C)\|x - \mu\|]$
\end{itemize}
and if
\[
n \geq (\log d)^{C} \cdot (k^2 + \log(1/\beta)) \cdot \frac 1 {\epsilon}\mcom
\]
then with probability at least $1-\beta$, the algorithm produces a vector $v$ such that $\left\| v - \tfrac{\mu - x}{\|\mu - x\|} \right\| \leq c$.
Furthermore, the algorithm runs in time $\poly(n,d, \log(1/\e),\log R)$.
\end{lemma}

The second lemma shows that there is a private procedure to estimate the distance from the current iterate $x$ to $\mu$, as long as $x$ is $k$-sparse.

\begin{lemma}[Estimating distance]\label{lem:estimating-distance}
  There exists a polynomial-time $\epsilon$-DP algorithm
  \texttt{Distance-Estimation} taking as input a vector $x\in\R^d$ with
  $\|x\|\leq R$ and independent samples $X_1,\ldots, X_n$ from $\cN(\mu, I)$,
  where $\|\mu\|\leq R$, with the following properties.
  For every $c > 0$ there is $C > 0$ such that given $n\geq \max\left\{\log(d)^C(k^2
  + \log(1/\beta))/\epsilon, 100 \log(R)\log(\log(R)/\beta)/\e \right\}$ samples, and if $\|x-\mu\|\geq C$ the algorithm outputs a
  distance estimate $\hat r$ such that $(1-c) \|x-\mu\| \leq \hat r \leq
  (1+c)\|x-\mu\|$ with probability at least $1-\beta$.
\end{lemma}

The next lemma shows that we can preserve the sparsity of our iterates by a simple thresholding procedure.

\begin{lemma}[Sparsifying iterates]\label{lem:sparsify-iterates}
  Let \texttt{Sparsify}$(x, k)$ denote the procedure of projecting $x$ onto its
  $k$ largest (in absolute value) coordinates, breaking ties at random.  Let
  $x,y \in \R^d$ with $\|y\|_0 \leq k$, and let $x'=$\texttt{Sparisfy}$(x,k)$.
  Then $\|x' - y\| \leq 4 \|x - y\|$.
\end{lemma}

The last lemma shows that the algorithm can detect when the current iterate $x$ is close enough to $\mu$.

\begin{lemma}[Deciding to halt]\label{lem:halt}
  There exists a universal constant $C$ such that for all $C'>C$ there exists a
  polynomial-time $\e$-DP algorithm $\texttt{Halt-Estimation}$
  with the following guarantees.
  For every $k$-sparse $x$ and $\mu$ such that $\|x\|,\|\mu\| \leq R$, given $x$ and $n \geq (\log d)^C (k^2 + \log(1/\beta))/\e$ samples from $\cN(\mu,I)$, \texttt{Halt-Estimation} outputs ``halt'' or ``do not halt''.
  If $\|x - \mu\| \leq C'$, then with probability at least $1-\beta$, \texttt{Halt-Estimation} outputs ``halt'', and if $\|x - \mu\| \geq 200C'$ then with probability at least $1-\beta$ \texttt{Halt-Estimation} outputs ``do not halt''.
  Furthermore, \texttt{Halt-Estimation} runs in time $\poly(n,d,\log(1/\e),\log R)$.
\end{lemma}

Now we can put together the pieces to prove \cref{thm:sos_sparse_mean_algo}.
\begin{proof}[Proof of \cref{thm:sos_sparse_mean_algo}]
First of all, by standard arguments, using buckets of size $b = 1/\alpha^2$ yields samples from $\cN(\mu, \alpha^2 I)$.
So by rescaling, it suffices therefore to prove the theorem in the case $\alpha = \Theta(1)$, which is what we will do.

For some $\beta'$ we choose momentarily, let us condition on all of the $1-\beta'$-probability events specified in Lemmas~\ref{lem:gradient},\ref{lem:estimating-distance}, and~\ref{lem:halt}, for all $O(\log R)$ calls to \texttt{Halt-Estimation}, \texttt{Distance-Estimation}, and \texttt{Gradient-Estimation}.
Choosing $\beta' = \Omega(\beta / \log R)$, so long as $n \geq (\log d)^{O(1)} \log R (k^2 + \log(1/\beta) + \log \log R) / \e$, by a union bound, with probability $1-\beta$ all these events occur.

Having conditioned on all these events, from~\cref{lem:halt}, it suffices to show that there exists $C > 0$ such that within $O(\log R)$ iterations of the main loop of \cref{alg:k2_logR_algo_sos} some iterate $\mu_i$ has $\|\mu_i - \mu\| \leq C$.
Then, we can choose $C'$ in Lemma~\ref{lem:halt} to have $C' \geq C$ and we obtain the guarantee that (a) \texttt{Halt-Estimation} outputs ``halt'' in some iteration and (b) when it does so, the output $\mu_{i}$ satisfies $\|\mu_i - \mu\| \leq O(1)$.

  It suffices therefore to show that there exists universal constant $C'>0$
  such that each iteration of gradient descent (lines 4-7
  in~\cref{alg:k2_logR_algo_sos}) shrinks $\|\mu_t -\mu\|$ by a factor of at least
  $0.9$, i.e. if $\|\mu_{t-1} - \mu\| \geq C'$ then $\|\mu_t-\mu\| \leq 0.9\|\mu_{t-1} -\mu\|$.
  Now we
  show that this follows from the guarantees of~\cref{lem:gradient}
  and~\cref{lem:estimating-distance}. In particular, it is enough to show
  that $\left\|g_t - \frac{\mu-\mu_{t-1}}{\|\mu-\mu_{t-1}\|} \right\|\leq 0.01$.
  Given that $(1-c)\|\mu_{t-1}-\mu\| \leq d_t \leq
  (1+c)\|\mu_{t-1}-\mu\|$ for a constant $c$ we will choose below (\cref{lem:estimating-distance}), before calling \texttt{Sparsify} we have that
  \begin{equation}
    \begin{split}
      \|\bar\mu_t-\mu\| &= \left\| \left(\mu_{t-1} -\mu + \eta\frac{\mu-\mu_{t-1}}{\|\mu-\mu_{t-1}\|}d_t\right) +
       \eta d_t g_t - \eta d_t \frac{\mu-\mu_{t-1}}{\|\mu-\mu_{t-1}\|} \right\| \\
      &\leq \left\|(\mu_{t-1} -\mu) - \eta\frac{\mu_{t-1} -\mu}{\|\mu-\mu_{t-1}\|}d_t\right\| + 0.01\eta d_t \\
      &\leq \left\| 1 - \frac{\eta d_t}{\|\mu_{t-1}-\mu\|} \right\| \|\mu_{t-1}-\mu\|  + 0.01(1+c)\eta \|\mu_{t-1}-\mu\|\\
      &\leq \left(\max\{|1-\eta(1+c)|, |1-\eta(1-c)|\} + 0.01(1+c)\eta \right) \|\mu_{t-1}-\mu\|.
    \end{split}
  \end{equation}
  Setting $\eta = 1$, we get that
  \begin{equation}
     \|\bar\mu_t-\mu\| \leq (1.01 c + 0.01)\|\mu_{t-1}-\mu\|.
  \end{equation}
  Choosing $c \leq 0.1$, and applying the guarantees of
  \cref{lem:sparsify-iterates} (\texttt{Sparsify} loses at most a factor of $4$)
  finishes the argument.
\end{proof}

%% file: appendicies/proofs_sos.tex
\section{Omitted Proofs from~\cref{sec:strongly-private-sparse}}
To set up, we define two key systems of polynomial inequalities.

\begin{definition}[The polynomial systems $\cQ$ and $\cP$]
  Let $\cQ$ be the following polynomial system in variables $v_1,\ldots,v_d,c_1,\ldots,c_d$:
  \[
  \|v\|^2 = 1, \, \|c\|^2 = 1, \, v_i v_j \leq c_i c_j \text{ and } -v_i v_j \leq c_i c_j \text{ for all $i,j \in [d]$, } \Paren{\sum_{i \in [d]} c_i}^2 \leq s\mper
  \]

For $X_1,\ldots,X_n \in \R^d$ and $x \in \R^d$ and $\alpha,s > 0$, we define the following system of inequalities $\cP(x,\delta,\alpha,s,X_1,\ldots,X_n)$ in variables $w_1,\ldots,w_d,v_1,\ldots,v_d, b_1,\ldots,b_n, c_1,\ldots,c_d$ to be:
\begin{align*}
  \cQ \cup \{ b_i^2 = b_i, b_i \iprod{X_i - x,v} \geq b_i \alpha \}_{i \in [n]} \cup \{ \|w-v\|^2 \leq \delta \}\mper
\end{align*}
\end{definition}

\subsection{Estimating Distance}
Here we prove \cref{lem:estimating-distance}.
\begin{proof}
  Let $X=\{X_i\}_{i\leq n}$ and let $f(r, X;\delta, s)= \max\pE\sum_{i\leq n}b_i$ s.t.
  $\pE$ satisfies $\cP(x, \delta, r, s, X)$; in words, $f(r, X)$ (roughly) counts the
  number of samples considered as ``inliers'' when we use radius $r$. The key
  idea is that this number is ``large'' when $r\gg\|x-\mu\|$ and ``small'' when
  $r\ll\|x-\mu\|$.

  More formally, we invoke the following private binary search
  primitive~\cite[Theorem 6.15]{hopkins2021efficient}: Given a function
  $f\colon[0, R]\times \cX^n\to [0,n]$, which is decreasing in its first
  argument, and has sensitivity $1$ in its second argument,
  $S\stackrel{\mathrm{def}}{=} \log(R/a)$ rounds of binary search suffice to
  output an estimate $\hat r\in[r_e-a, r_s+a]$, where $r_e=\min r$ s.t. $f(r,
  X)\leq ek + \Delta$, and $r_s=\max r$ s.t.  $f(r, X)\geq sk - \Delta$ for
  $\Delta=S\log(S/\beta)/\e$.  We choose $s=1-2\delta, e=1-\delta, \Delta =
  \frac{\delta}{10}n$ for some $\delta > 0$.

  First, monotonicity follows from the fact that the solution for a given
  $r=r'>r''$ is a feasible solution for $r=r''$, and bounded sensitivity follows
  from the fact that the value of the (non-relaxed) quadratic program has
  sensitivity $1$ since we can change at most one of the indicators upon
  changing a single sample; furthermore, there is an SoS proof of
  that~\cite{hopkins2021efficient}.

  Now we turn to quantifying $r_e$ and $r_s$; we show the following two facts:
  \begin{enumerate}
    \item \label{turn_left} For all $c_e>0, \delta > 0$, there exists $C_e>0$
    such that if $\|x-\mu\|\geq C_e$, then $r_e\geq (1-c_e)\|x-\mu\|$.
    \item \label{turn_right} For all $c_s>0, \delta > 0$, there exists $C_s>0$
    such that if $\|x-\mu\|\geq C_s$, then $r_s\leq (1+c_s)\|x-\mu\|$.
  \end{enumerate}

  For Fact~\ref{turn_left}, it suffices to show that $f((1-c_e)\|x-\mu\|, X)\geq
  (1-\delta) n$. Now, it is enough to exhibit a feasible direction. In particular,
  set $v=\frac{\mu-x}{\|\mu-x\|}$.  Let $S\subseteq[n]$ be a subset of samples
  such that $|S|\geq (1-\delta)n$ and $\cQ \proves_{O(1)} \sum_{i\in S}
  \iprod{X_i-\mu, v}^2 \leq 2n$. We know such a subset exists from
  \cref{lem:covariance-subset} as long as $n\geq \log(d)^C\left(k^2 +
  \log(1/\beta)\right)/\epsilon$. Then we have that there exists $S'\subseteq S$
  such that $|S'|\geq (1-2\delta)n$ for which we have $|\iprod{X_i
  -\mu, v}|\leq 1/(2\delta)$ for all $i\in S'$. From this, we know that
  \begin{equation}
    \sum_{i\in S'}\iprod{X_i-x, v} =
    \sum_{i\in S'}\iprod{X_i-\mu, v} + \sum_{i\in S'}\iprod{\mu - x, v}
    \geq  - \frac{1}{2\delta} + \|x-\mu\|.
  \end{equation}
  Now we need $\|x-\mu\|\geq \frac{1}{2c_e\delta}$ for Fact~\ref{turn_left} to
  hold. Setting $C_e$ to the RHS finishes the argument.

  For Fact~\ref{turn_right}, it is enough to show that $f((1+c_s)\|x-\mu\|, X)\leq
  (1-2\delta) n$.
  Fix $\delta'$ and $L$ to be chosen later.
  Let $S\subseteq[n]$ be a subset of samples such that $|S|\geq
  (1-\delta')n$ and $\cQ \proves_{O(1)} \sum_{i\in S} \iprod{X_i-\mu, v}^2 \leq 2n$
  (such exists from \cref{lem:covariance-subset}).

  Note that $(a+b)^2 \leq (1+L^2)a^2+(1+1/L^2)b^2$ is true for any $L\in\R$ and
  has an SoS proof. Using this and the SoS Cauchy-Schwarz inequality, we have the following sequence of inequalities, all with degree-$2$ SoS proofs.
  \begin{equation}
    \begin{split}
      \left(\sum_{i\leq n} b_i\right)^2&\leq \left((n-|S|) + \sum_{i\in S} b_i\right)^2
      \leq\left((n-|S|) + \sum_{i\in S}b_i \frac{\iprod{X_i - x, v}}{r} \right)^2\\
      &\leq \Paren{1+L^2}(n-|S|)^2 + \left(1 + \frac{1}{L^2}\right)\Paren{ \sum_{i \in S} b_i \frac{\iprod{X_i - x, v}}{r} }^2\\
      &\leq \Paren{1+L^2}(\delta'n)^2 + \left(1 + \frac{1}{L^2}\right)\Paren{ \sum_{i \in S} b_i \frac{\iprod{X_i - \mu, v} + \iprod{\mu-x, v}}{r} }^2\\
      &\leq \Paren{1+L^2}(\delta'n)^2 + \\
      & \ \ \ \ \ \Paren{1+\frac{1}{L^2}}\Paren{ \Paren{1+L^2}
       \Paren{\sum_{i \in S} b_i \frac{\iprod{X_i - \mu,v}}{r}}^2 +
       \Paren{1 + \frac{1}{L^2}} \Paren{\sum_{i \in S} b_i \frac{\iprod{\mu-x,v}}{r}}^2 }\\
      & \leq \Paren{1+L^2}(\delta' n)^2 +
       \Paren{1 + \frac{1}{L^2}}\Paren{1+L^2} \Paren{\sum_{i\in S}b_i^2} \frac{2n}{r^2} + \\
      & \ \ \ \ \ \Paren{1+\frac{1}{L^2}}^2 \Paren{\sum_{i\in S}b_i^2}
        \Paren{ \sum_{i\in S} \frac{\|\mu-x\|^2}{r^2} }\\
      &\leq n^2\cdot \Paren{\Paren{1+L^2}(\delta')^2 +
       \Paren{L^2 + 2 + \frac{1}{L^2}} \frac{|S|}{n} \frac{2}{r^2} +
       \Paren{1+\frac{1}{L^2}}^2  \frac{|S|^2}{n^2} \frac{\|\mu-x\|^2}{r^2}}\\
      &\leq n^2\cdot \Paren{\Paren{1+L^2}(\delta')^2 +
       \Paren{2L^2 + 4 + \frac{2}{L^2} +
       \Paren{1+\frac{1}{L^2}}^2 \|\mu-x\|^2} \frac{1}{r^2}}
    \end{split}
  \end{equation}
  We can choose $\delta' \leq \frac{\sqrt{1-2\delta}}{100\sqrt{1 + L^2}}$ such
  that the first term is sufficiently small. Next, we choose $L$ so that the
  overall sum is at most $(1-2\delta)n$. %for any $\delta \ll c$.
  In particular, we want that
  \begin{equation}
       \frac{2L^2 + 4 + \frac{2}{L^2}}{(1+c_s)^2\|\mu-x\|^2} +
       \Paren{\frac{1+1/L^2}{1+c_s}}^2
       \leq \frac{2L^2 + 4 + \frac{2}{L^2}}{(1+c_s)^2C_s^2} +
       \Paren{\frac{1+1/L^2}{1+c_s}}^2
       \leq (1-2\delta)^2.
  \end{equation}
  Thus, we can choose $L$ such that $L \geq
  \frac{1}{\sqrt{c_s-2\delta(1+c_s)}}$ and $2L^2+4+2/L^2\leq
  \Paren{(1-2\delta)(1+c_s)C_s}^2$ (solving self-consistently).

  Combining the conclusion of~\cite[Theorem 6.15]{hopkins2021efficient} with
  the conclusions of Facts~\ref{turn_left} and~\ref{turn_right}, we receive as output
  a distance estimate $\hat r\in[(1-c_e)\|x-\mu\| - a, (1+c_s)\|x-\mu\| + a]$.
  Noting that $\|x-\mu\|\geq \max\{C_e, C_s\}$, we can set $c_s=c_e=c/2$ and
  $a=\max\{C_e, C_s\} \cdot c/2$ to obtain the desired result.

  \end{proof}

\subsection{Sparsifying}
\begin{lemma}[\cref{lem:sparsify-iterates} restated]
  Let \texttt{Sparsify}$(x, k)$ denote the procedure of projecting $x$ onto its
  $k$ largest (in absolute value) coordinates, breaking ties at random.  Let
  $x,y \in \R^d$ with $\|y\|_0 \leq k$, and let $x'=$\texttt{Sparisfy}$(x,k)$.
  Then $\|x' - y\| \leq 4 \|x - y\|$.
 \end{lemma}
\begin{proof}
  Let $S=\mathrm{supp}(y)$ and $T=\mathrm{supp}(x')$. Now note that
  $|S\setminus{T}|=|T\setminus S|\leq k$. Additionally, from the assumption that
 we select the largest $k$ components, we have that $\|x_{T\setminus
  S}\|\geq \|x_{S\setminus T}\|$.

  Using the above observation, together with the triangle inequality, we have
  that
  \begin{equation}
    \begin{split}
    \|x'-y\| &\leq \|x_{S\setminus T}'-y_{S\setminus T}\| +
                   \|x_{T\setminus S}'-y_{T\setminus S}\| +
                   \|x-y\| \\
    &\leq \|x_{S\setminus T}'- x_{S\setminus T}\| +
          \|x_{S\setminus T} - y_{S\setminus T}\| +
       \|x_{T\setminus S}-y_{T\setminus S}\| + \|x-y\| \\
    &\leq \|x_{S\setminus T}\| + 3\|x-y\| \\
    &\leq \|x_{T\setminus S}\| + 3\|x-y\| \\
    &\leq 4\|x-y\|,
    \end{split}
  \end{equation}
  as desired.

\end{proof}

\subsection{Deciding to Halt}
In this section we prove \cref{lem:halt}.
\begin{proof}[Proof of \cref{lem:halt}]
  The algorithm is as follows.
  Compute $Z\stackrel{\mathrm{def}}{=}\max \pE \sum_{i=1}^n b_i$ over all $\pE$
  satisfying $\cP(x, 1, 100C', X_1,\ldots,X_n)$ and add noise $Lap(1/\e)$.  If
  the result is $\leq 0.1n$, output ``halt'', otherwise output ``do not halt''.

  Sensitivity of $Z$ was already proved in~\cref{lem:estimating-distance}.
  Privacy follows directly from the guarantees of the Laplace mechanism.

  Next we establish correctness. This also follows directly from the arguments
  in \cref{lem:estimating-distance} since we can view halt estimation as a
  one-step binary search. In particular, let $S\subseteq [n]$ be a
  subset such that $|S|\geq (1-\delta)n$ and $\cQ\proves \sum_{i\in
  S}\iprod{X_i-x,v}^2 \leq 2n$.
  Then we claim that the algorithm above will output ``halt'' with probability
  at least $1-\beta$, using that $C'$ is at least some universal constant.

  At the same time, suppose that the algorithm above outputs ``halt''.
  Conditioning again on an event of probability $1-\beta$, this means that $\max \pE \sum_{i=1}^n b_i$ over all $\pE$ satisfying $\cP(x, 1, 100C', X_1,\ldots,X_n)$ is at most $0.2n$.
  In particular, for every $2k$-sparse unit vector $v$, at least $0.8n$ choices of $i$ have $\iprod{X_i - x,v} \leq 100C'$, and the same is true for $\iprod{X_i - \mu, v}$.
  Choosing some $i$ such that this holds simultaneously for $v = (\mu-x)/\|\mu -x\|$ in the first case and $v = (x-\mu)/\|x - \mu\|$ in the second, we obtain by adding the two,
  \[
  200C' \geq \iprod{X_i - x, (\mu - x)/ \|\mu - x\|} - \iprod{X_i - \mu, (\mu-x)/\|\mu - x\|} = \|\mu - x\|\mper
  \]
\end{proof}

\subsection{Finding Gradients}
In this section we prove \cref{lem:gradient}.
The promised algorithm will use the SoS exponential mechanism of
\cite{hopkins2021efficient}.

Our first lemma says that the quadratic form of the covariance matrix of
$X_1,\ldots,X_n$, after throwing out a few samples, has bounded quadratic form
in sparse directions, and that furthermore this has an SoS proof.  

\begin{lemma}\label{lem:covariance-subset}
  For every $c > 0$ there exists a universal constant $C > 0$ such that for all
  large-enough $d \in \N$, all $s \leq d$, and all $\beta \in (0,1)$, for all $n
  \geq (\log d)^C (s^2 + \log(1/\beta))$, with probability at least $1-\beta$
  over i.i.d. $X_1,\ldots,X_n \sim \cN(0,I)$ there is a subset $S \subseteq [n]$
  with $|S| \geq (1-c)n$ and such that 
  \[
  \cQ \proves_{O(1)} \sum_{i \in S} \iprod{X_i,v}^2 \leq 2n \mcom
  \]
\end{lemma}

The next lemma gives the SoS proof of utility we need for SoS exponential mechanism.

\begin{lemma}\label{lem:sos-utility}
  For every $c > 0$ there is $C > 0$ such that the following holds.
  Let $X_1,\ldots,X_n \sim \cN(\mu,I)$ for $\mu \in \R^d$ with $\|\mu\|_0 \leq k$, with $d \in \N$ sufficiently large.
  For all $x \in \R^d$ with $\|x - \mu\| > C$, all $\beta \in (0,1)$, and $r \in [(1-1/C)\|x - \mu\|, (1+1/C) \|x - \mu\|]$, and $n \geq (\log d)^C (k^2 + \log(1/\beta))$, with probability at least $1-\beta$ over $X_1,\ldots,X_n$, there is a degree-$O(1)$ SoS proof
  \[
  \cP(x,1/C, r, k/c , X_1,\ldots,X_n), \sum_{i \leq n} b_i \geq (1-1/C)n \proves_{O(1)} \left \| w - \frac{\mu - x}{\|\mu - x\|} \right \|^2 \leq c\mcom
  \]
  and, furthermore, this proof has degree-$1$ in the constraint $\sum_{i \leq n} b_i \geq (1-1/C)n$.
\end{lemma}

The last lemma proves key properties of the convex set over which SoS exponential mechanism will run a log-concave sampling algorithm.

\begin{lemma}\label{lem:l1-fattening}
  Let $d,k \in \N$ with $k \leq d$, and let $\delta > 0$.
  Then the set
  \[
  \cC(k,\delta) = \{ x \in \R^d \, : \, \exists y \in \R^d \text{ s.t. } \|y\|_2 \leq 1, \|y\|_1 \leq \sqrt{k}, \|x - y\|_2 \leq \delta \}
  \]
  has the following properties:
  \begin{itemize}
  \item $\cC$ is compact and convex.
  \item $\cC$ has diameter $O(1+\delta)$.
  \item $\cC$ admits polynomial-time projection and membership oracles.
  \item For every $v \in \R^d$ with $\|v\| \leq 1$ and $\|v\|_1 \leq \sqrt{k}$, there exists a set $S_v \subseteq \cC$ such that for all $w \in S_v$, $\|w - v \| \leq \delta$, and $|S_v| \geq \exp(-O(d / (\delta \sqrt k) + k^2 \log d)) \cdot |\cC|$, where $|\cdot|$ denotes Lebesgue measure.
  \end{itemize}
\end{lemma}

Now we can apply Theorem 4.5 of \cite{hopkins2021efficient} to prove \cref{lem:gradient}.

\begin{proof}[Proof of \cref{lem:gradient}]
  We verify that the conditions of Theorem 4.5 of \cite{hopkins2021efficient}
  apply to our polynomial system $\cP$, with $p(b) = \sum_{i \in [n]} b_i -
  (1-1/C)n$ for a large-enough constant $C$, and the convex set $\cC(k,1/C)$ of
  \cref{lem:l1-fattening}. 
  \begin{itemize}
  \item Compactness and convexity of $\cC$, projection and membership oracle,
  diameter at most $\poly(n,d)$: guaranteed by \cref{lem:l1-fattening}. 
  \item $\cP$ is Archimedian: $\cP$ contains constraints upper-bounding
  $\|v\|^2, \|c\|^2, \|b\|^2$ by $\poly(d,n)$, by inspection.
  \item Robust satisfiability: let $\eta = 1/\poly(n,d)$.
  We claim that for all $X_1,\ldots,X_n,\alpha$, the system
  $\cP(x,2/C,\alpha,s,X_1,\ldots,X_n)$ is $\eta$-robustly satisfiable with
  respect to $\cC$ and $p(b) = \sum_{i \leq n} b_i - (1-1/C)n$. 

  To see this, consider any $y \in \cC$ and any $y'$ such that $\|y'-y\| \leq 1/\poly(d,n)$.
  Then, setting $w = y'$, we can satisfy $\cP$ by choosing $v$ to be the element of $\cC$ satisfying $\|v - y\| \leq 1/C$ and having $\|v\|_2 \leq 1, \|v\|_1 \leq \sqrt{k/c}$; then we will have $\|w - v\| \leq 1/C + 1/\poly(d,n) \leq 2/C$.
  Choosing $c_i = |v_i|$, and choosing $b_i = 0$ for all $i$ satisfies the rest of the constraints in $\cP$.
  \item SoS proof of bounded sensitivity: same as Example 4.2 in \cite{hopkins2021efficient}.
  \item SoS proof of utility: existence with probability at least $1-\beta$ over choice of samples follows directly from \cref{lem:gradient}.
  This proof if expressible in at most $\poly(d,n,\log R)$ bits by inspection of the proof of \cref{lem:gradient}.
  \item Volume ratio: We claim that for $n \geq (\log d)^C (k^2 + \log(1/\beta))$ and $k \geq d^{0.4}$, there exists a set $S \subseteq \cC$ with $|S|/|\cC| \geq d^{-O(k^2)}$ (where $|\cdot|$ denotes Lebesgue measure) such that for all $z \in S$ there is a solution $v,w,b,c$ to $\cP$ with $p(b) \geq \Omega(n)$ and $w=z$.

  To see this, we take $S$ to be the set of $z$ such that $\|z-
  (\mu-x)/\|\mu-x\|\| \leq 1/C$.  First, $\|\mu-x\|_0 \leq 2k$, so $\|\mu
  - x\|_1 \leq \sqrt{2k} \|\mu-x\|_2$, hence $S \subseteq \cC$.  For any $z \in
  S$ we can set $w = z$ and $v = (\mu-x)/\|\mu-x\|$, then $\|w-v\|^2 \leq 2/C$.
  Take $c_i = |v_i|$; since $v$ is $2k$-sparse this satisfies the constraints of
  $\cQ$. Take $b_i = 1$ if $\iprod{X_i - x, \tfrac{\mu-x}{\|\mu-x\|}} =
  \iprod{X_i - \mu, \tfrac{\mu-x}{\|\mu-x\|}} + \|\mu-x\| \geq r$, again
  satisfying $\cP$. 

  It remains to show that $\sum_{i \in [n]} b_i - (1-1/C)n \geq \Omega(n)$.
  Applying \cref{lem:covariance-subset} and Markov's inequality, for large-enough $n \gg (\log d)^{C'}(k^2 + \log(1/\beta))$ for some other constant $C'$, there is a set $T \subseteq [n]$ with $|T_i| \geq (1-1/(10C))n$ and $|\iprod{X_i-\mu, (\mu-x)/\|\mu-x\|}| \leq O(1)$.
  So, as long as $r \leq \|\mu-x\| - O(1)$, $\sum b_i$ is large enough.
  \end{itemize}
  It follows that for fixed $x$ and $\mu$, with probability at least $1-\beta$ over $n \geq (\log d)^C (k^2 + \log(1/\beta))/\epsilon$ samples from $\cN(\mu,I)$, SoS exponential mechanism is $\e$-DP and outputs a vector $v$ such that $\|v - (\mu-x)/\|\mu-x\| \leq c$.

  For the running time dependence on $\log(1/\e)$, note that the dependence of $1/\e$ claimed in \cite{hopkins2021efficient} for SoS exponential mechanism can be improved by appeal to  the main result of \cite{mangoubi2021sampling}.
\end{proof}

\subsubsection{Proof of \cref{lem:sos-utility},~\cref{lem:covariance-subset}}

\begin{proof}[Proof of \cref{lem:covariance-subset}]
  Consider the random variable $B = \max \{ |S| \, : \, \cQ \proves_{O(1)} \sum_{i \in S} \iprod{X_i,v}^2 \leq 2n \}$.
  By McDiarmid's bounded-differences inequality, $\Pr ( |B - \E B| > t ) \leq 2 \exp(-\Omega(t^2 / n))$, so as long as $n \gg \log(1/\beta)$, we have $|B - \E B| \leq cn/10$ with probability at least $1-\beta$.
  It will suffice therefore to show that $\E B \geq (1-c/10)n$.
  For this in turn it suffices to show that with probability at least $1-c/20$ we have $B \geq (1-c/20)n$.

  Let $S \subseteq [n]$ be the set of $X_i$ such that $\|X_i\|_{\infty} \leq c' \sqrt{\log d}$, with $c' > 0$ chosen so that $|S| \geq (1-c/20)n$ with probability at least $(1-c/100)$.
  We claim that if $n \gg \poly \log (d) \cdot s^2$ then $\cQ \proves_{O(1)} \sum_{i \in S} \iprod{X_i,v}^2 \leq 2n$ with probability at least $1-c/100$; then a union bound finishes the proof.

   Let $M = \sum_{i \in S} X_i X_i^\top$.
  By Bernstein's inequality, with probability at least $1-\delta$ the following both hold:
  \begin{align*}
  & \max_{a \neq b \in [d]} |M_{ab}| \leq O(\sqrt{n \log(1/\delta)} + \log d \cdot \log(1/\delta) ) \\
  & \max_{a \in [d]} |M_{aa}| \leq n + O(\sqrt{n \log(1/\delta)} + \log d \cdot \log (1/\delta))\mper
  \end{align*}
  For all $a,b \in [d]$, $\cQ \proves_{O(1)} v_a v_b M_{ab} \leq |M_{ab}| c_a c_b$,
  so
  \begin{align*}
    \cQ \proves \sum_{i \leq n} \iprod{X_i,v}^2 & \leq \max_{a \in [d]} |M_{aa}| \sum_{a  \in [d]} c_a^2 + \max_{a \neq b \in [d]} |M_{ab}| \cdot \Paren{\sum_{a \in [d]} c_a }^2\\
    & \leq \max_{a \in [d]} |M_{aa}| + s \cdot \max_{a \neq b \in [d]} |M_{ab}| \\
    & \leq s \cdot O(\sqrt{n \log(1/\delta)} + \log d \cdot \log (1/\delta)) + n \mper
  \end{align*}
  Choosing $\delta$ a small enough, and then $C$ large enough, completes the proof.
\end{proof}

\begin{lemma}
\label{lem:remove-square}
For an indeterminate $X$ and $C > 0$,
\[
X^2 \leq C^2 \proves_2 X \leq 2C\mper
\]
\end{lemma}
\begin{proof}
  Note that $X = \frac 1 {2C} ( (X + C)^2 - X^2 - C^2)$.
So $X^2 \leq C^2 \proves X \leq \frac 1 {2C} (X+C)^2 \leq \frac 1 {2C} \cdot 2 (X^2 + C^2)$ via SoS triangle inequality.
Using the axiom $X^2 \leq C^2$ completes the proof.
\end{proof}

Now we can prove \cref{lem:sos-utility}.

\begin{proof}[Proof of \cref{lem:sos-utility}]
  First, by SoS triangle inequality and the constraint $\|w-v\|^2 \leq 1/C$, for large-enough $C$, it suffices to show that
  \[
  \cP(x,1/C, r, k, X_1,\ldots,X_n), \sum_{i \leq n} b_i \geq (1-1/C) n \proves_{O(1)} \left \| v - \frac{\mu - x}{\|\mu - x\|} \right \|^2 \leq c/4\mper
  \]
  And, for this, it is enough to show $\cP, \sum_{i \leq n} b_i \geq (1-1/C)n \proves_{O(1)} \iprod{\mu-x,v} \geq (1-c/8)\|\mu - x\|$.
  (All these proofs must be degree-$1$ in $\sum_{i \leq n} b_i \geq (1-1/C)n$, which will be true by construction.)

  Let $S \subseteq [n]$ be the set of indices guaranteed to exist by \cref{lem:covariance-subset}, of size $(1-1/C)n$, for the vectors $X_1-\mu,\ldots,X_n-\mu$.
  We first claim that
  \[
  \cP \proves_{O(1)} \sum_{i \in S} b_i \iprod{X_i - \mu,v} \leq 4n \mper
  \]
  For this, using \cref{lem:remove-square}, we can show instead
  \[
  \cP \proves_{O(1)} \Paren{\sum_{i \in S} b_i \iprod{X_i - \mu,v}}^2 \leq 2n^2\mcom
  \]
  but this follows immediately from SoS Cauchy-Schwarz and \cref{lem:covariance-subset}.

  Now,
  \[
    \begin{split}
  \cP, \sum_{i \in n} b_i \geq (1-1/C)n \proves_{O(1)}& r \cdot (1-2/C)n \leq r \cdot \sum_{i \in S} b_i  \\
  &\leq \sum_{i \in S} b_i \iprod{X_i - x,v} = \sum_{i \in S} b_i \iprod{X_i - \mu,v} + b_i \iprod{\mu - x,v}\mper
    \end{split}
  \]
  Putting this together with the preceding, we have
  \[
  \cP, \sum_{i \in [n]} b_i \geq (1-1/C)n \proves_{O(1)} (1-2/C) r n - 4n \leq \iprod{\mu-x,v} \cdot \sum_{i \in S} b_i = n \cdot \iprod{\mu - x,v} + \Paren{\sum_{i \in S} b_i - n} \cdot \iprod{\mu-x,v}\mper
  \]
  Next we claim that $\cP, \sum_{i \in [n]} b_i \geq (1 - 1/C)n \proves_{O(1)} \Paren{\sum_{i \in S} b_i - n} \cdot \iprod{\mu-x,v} \leq O(1/\sqrt{C}) n \|\mu -x\|$.

  We have
  \begin{align*}
  \|v\|^2 \leq 1 \proves_{O(1)}& \Paren{\sum_{i \in S} b_i - n} \cdot \iprod{\mu-x,v} \\
 & \leq \frac {\sqrt{C} \|\mu-x\|} n \cdot \Paren{\sum_{i \in S} b_i -n }^2 + \frac{ n}{\sqrt{C} \|\mu-x\|} \iprod{\mu-x,v}^2\\
 & \leq \frac {\sqrt{C} \|\mu-x\|} n \cdot \Paren{\sum_{i \in S} b_i -n }^2 + (1/\sqrt{C}) n \|\mu-x\|\mper
 \end{align*}
 Now, $\{b_i^2 = b_i\}_{i \in [n]} \proves_{O(1)} \Paren{\sum_{i \in S} b_i - n}^2 \leq 2n \Paren{\sum_{i \in S} b_i - n}$, so via a proof which is degree-$1$ in $\sum_{i \in [n]} b_i \geq (1-1/C)n$ we have
 \[
 \cP, \sum_{i \in [n]} b_i \geq (1-1/C)n \proves_{O(1)} \Paren{\sum_{i \in S} b_i - n}^2 \leq 2 \cdot 1/C \cdot n^2\mper
 \]
 Putting everything together and using the assumption $r \geq (1-1/C) \|\mu-x\|$, we get
 \[
 \cP, \sum_{i \leq n} b_i \geq (1-1/C) n \proves_{O(1)} (1-O(1/\sqrt{C})) \|\mu-x\| - 4 \leq \iprod{\mu-x,v}
 \]
 which gives the conclusion by taking $C = C(c)$ large enough.
\end{proof}

\subsubsection{Proof of \cref{lem:l1-fattening}}

\begin{proof}[Proof of \cref{lem:l1-fattening}]
  Let $\cC = \{x \in \R^d \, : \, \exists v \in \R^d \text{ s.t. } \|x - v\| \leq \delta \text{ and } \|v\|_1 \leq \sqrt{k}, \|v\|_2 \leq 1 \}$ be a ``fattening'' of the scaled $\ell_1$ ball.
  $\cC$ is convex by inspection.

  To compute membership in $\cC$, given $x$, compute the projection $x'$ to $\{ v \, : \, \|v\|_1 \leq \sqrt{k} \text{ and } \|v\|_2 \leq 1\}$.
  If $\|x - x'\| \leq \delta$ then $x \in \cC$ and otherwise $x \notin \cC$.
  Furthermore, if $\|x - x'\| > \delta$, then for some $\delta'$, a hyperplane through $x + \delta'(x' - x)$ separates $x$ from $\cC$.
  With this separation oracle for $\cC$, projections can be computed by minimizing $\|x - y\|$ over $y \in \cC$.

  Now we move on to the volume arguments.
  By Sudakov minoration (see e.g. \cite{ledoux1991probability}), $\{ v \: \, \|v\|_1 \leq \sqrt{k} \}$ can be covered by $d^{O(k^2)}$ $\ell_2$ balls of radius $1/\sqrt{k}$.
  So, $\cC$ can be covered by $d^{O(k^2)}$ $\ell_2$ balls of radius $1/\sqrt{k} + \delta$.
  For any $v$ with $\|v\|_1 \leq \sqrt{k}$ and $\|v\|_2 \leq 1$, we know $S_v = \{ w \, : \, \|w - v\|\leq \delta\} \subseteq \cC$.
  Furthermore,
  \[
  \frac{|S_v|}{|\cC|} \geq d^{-O(k^2)} \cdot \Paren{\frac{\delta}{\delta + \tfrac 1 {\sqrt k}}}^d  = d^{-O(k^2)} 2^{-O(d / (\sqrt{k} \delta))}\mper
  \]
\end{proof}

%% file: appendicies/experiment_details.tex
\section{Fast Algorithms and Experiments}\label{sec:experimental-details}

\begin{algorithm}[h!]
\textbf{Input:} $\{x_i\}_{i\leq n}$, $T$, $b$, $\sigma^2$, $n$; 
$\hat\mu \gets \vec{0}$, selected $\gets$ []\\
\begin{algorithmic}[1]
\State $m_j \gets \frac{1}{b}\sum_{j=i\cdot b}^{(i+1)b} x_i$ for
$j\in[\lceil n/b\rceil]$ 
\Comment{compute bucketed means}
\State $z_i \gets \sum_{j=1}^{\lceil n/b\rceil} \1\{(m_j)_i \geq T\}$ for
$i\in[d]$.
\Comment{coordinate-wise threshold}
\For{$j \leq k$}
    \State $t\gets$ \texttt{exp-mech}$\left(\{z_i\}_{i\in [d]\setminus \text{selected}} \right)$
    \State selected.insert($t$)
    \State $\hat\mu_t\gets$ \texttt{KV-1D}$\left(\{(m_i)_t\}_{i=1}^d\right)$
    \Comment{univariate estimation}
\EndFor
\end{algorithmic}
\textbf{Output:} mean estimate $\hat\mu$
\caption{The subroutine \texttt{exp-mech}
refers to the exponential mechanism~\cite{mcsherry2007mechanism}, and 
\texttt{KV-1D} --- to the univariate sparse mean estimator
of~\cite{karwa2017finite}.}\label{alg:k2_logR_algo_peeling} 
\end{algorithm}

\begin{theorem}\label{thm:sparse-alg-peeling}
    Let $X_1,\ldots, X_n\sim\cN(\mu, \sigma^2)$.
    \cref{alg:k2_logR_algo_peeling} is $\epsilon$-DP outputs an
    estimate $\hat\mu$ s.t. $\|\hat\mu-\mu\|_2\leq \alpha$ with probability at
    least $1-\beta$ as long as
    \begin{equation}
        \begin{split}
    n &= \Omega\left(
   \underbrace{\frac{k^2\sigma^2\left(\log d + \log(1/\beta)\right)}{\alpha^2\epsilon}}_{\text{support estimation}} +
   \underbrace{\frac{\sigma^2k\log(2k/\beta)}{\alpha^2} + \frac{\sigma k^{1.5}\log(2k/\beta)}{\alpha\epsilon} + \frac{k\log(R)}{\epsilon}}_{\text{dense estimation in }k\text{ dimensions}} \right)
        \\    
         &= \Omega\left(\sigma^2\frac{k^2(\log d+\log(k/\beta))}{\alpha^2\epsilon} +
                        \frac{k\log(R)}{\epsilon}
                  \right).
        \end{split}
    \end{equation}
\end{theorem}

The proof is given in~\cref{sec:proof_peeling}. While the statement and proof
are given for Gaussians, we only use  the concentration properties of the
Gaussian distribution; hence the proof automatically applies for sub-Gaussian
distributions as well. 

Instead of the exponential mechanism, we could also use the
peeling~\cite{dwork2018differentially} algorithm, as is done for the
linear-in-$R$ algorithm in~\cite{cai2021cost}.

\subsection{Experimental details}

Now we turn to empirically validating the performance
of~\cref{alg:k2_logR_algo_peeling}. To isolate the effects of each subroutine
(support estimation and dense mean estimation), we plot (1) the performance of
the corresponding support estimation steps alone (2) $\ell_2$ error for both. 

\begin{figure}[h!]
    \centering
    \includegraphics[width=0.6\textwidth]{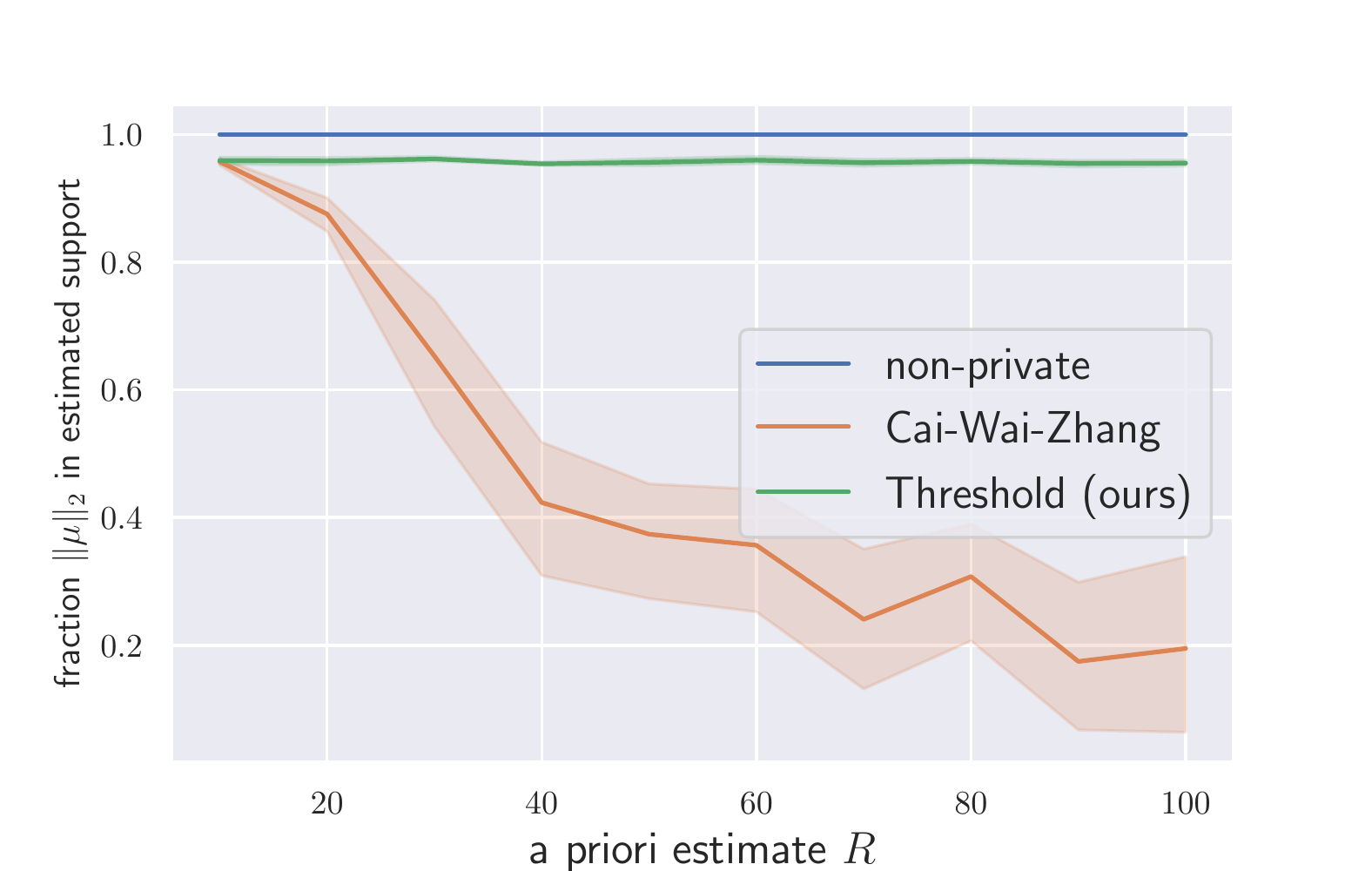}
    \caption{Empirical evaluation of $0.5$-DP algorithms (and a
    non-private baseline) for
    support estimation for $1500$ samples from $\cN(\mu, I)$ in ambient
    dimension $d=1000$ with $\|\mu\|_0=20$ (non-zero coordinates sampled
    uniformly from $[-10, 10]$) as a function of $R$, the a priori
    estimate of $\|\mu\|_2$.
    }\label{fig:experimental_results} 
\end{figure}

\begin{figure}[h!]
    \centering
    \includegraphics[width=0.6\textwidth]{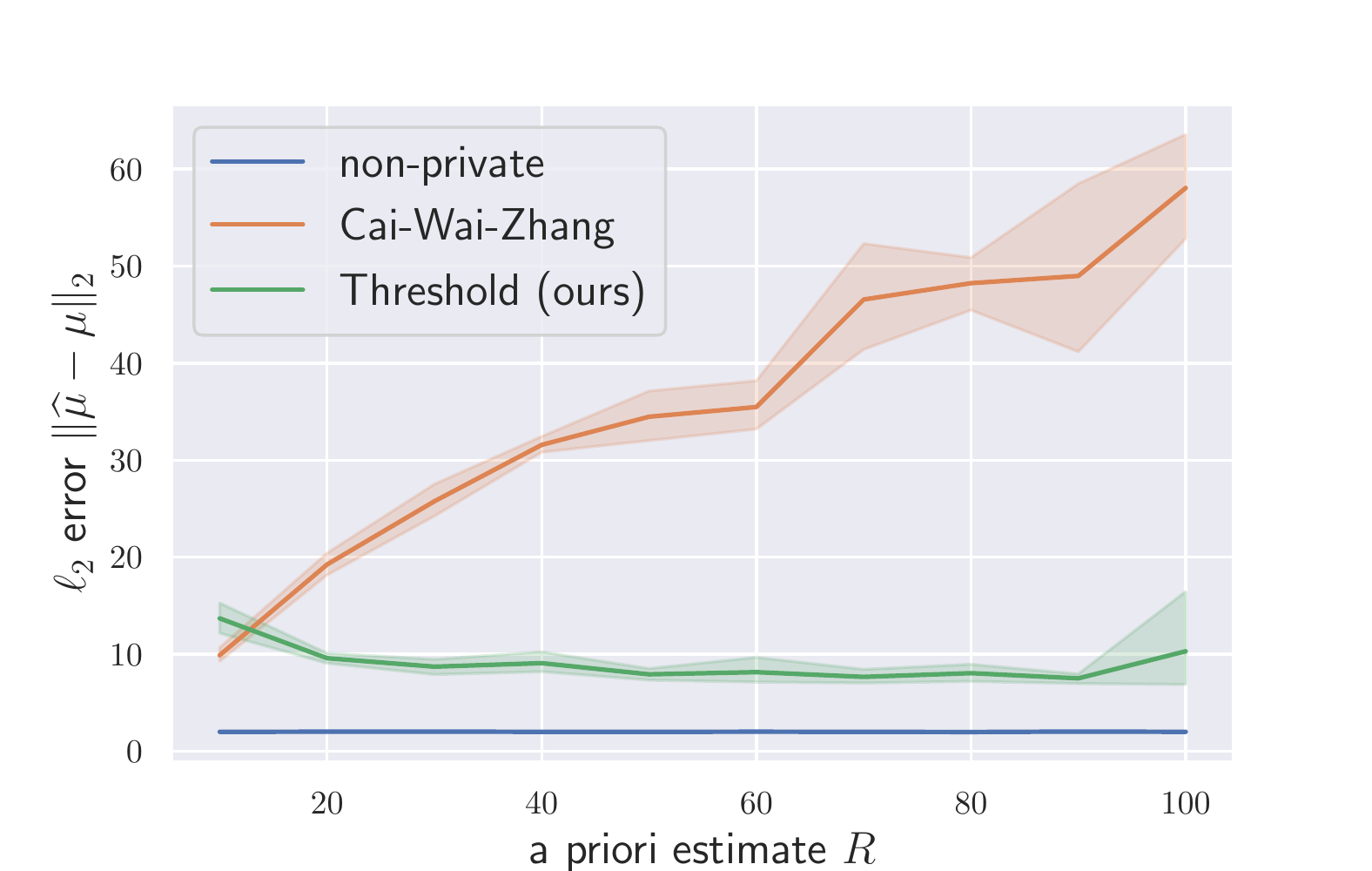}
    \caption{Empirical evaluation of \texttt{Threshold} and the sparse mean
    estimation algorithm of~\cite{cai2021cost} under $\epsilon$-DP with
    $\epsilon=0.5$; results are shown for Gaussian data $X_1,\ldots,
    X_n\sim\cN(\mu, 4\cdot I)$ for a $k$-sparse $\mu$ in $\R^d$ for $k=20,
    d=1000, n=1000$; non-sparse coordinates of $\mu$ are sampled uniformly from
    the interval $[-10, 10]$. The $\ell_2$ error of the estimates is plotted
    against the a priori mean estimate $R$. A folklore non-private baseline is
    also presented to highlight the cost of privacy.} \label{fig:l2_estimate}  
\end{figure}

We address (1) in~\cref{fig:experimental_results} ---  we use the
fraction ($\ell_2$) mass of $\mu$ on the $k$ coordinates the algorithms select
as a metric of success, since it is well-suited for mean estimation.
\cref{fig:experimental_results} shows that (a) we significantly outperform the
previous state of the art as soon as we introduce very mild uncertainty in the a  
priori estimate of $\|\mu\|$; (b) our method does not introduce additional
constant factors ``hiding'' in the asymptotics.

In~\cref{fig:experimental_results} we presented results for support estimation,
in order to highlight the improvement coming from coordinate selection alone.
In~\cref{fig:l2_estimate} we evaluate~\cref{alg:k2_logR_algo_peeling} directly
in terms of our metric of interest -- $\ell_2$ error.

As can be seen from~\cref{fig:l2_estimate}, the performance of the sparse mean
estimation algorithm of~\cite{cai2021cost} degrades rapidly even for very mild
levels of uncertainty in the range of the mean - e.g. if we only know beforehand
that the mean lies in the range $[-20, 20]$, instead of the tight range $[-10,
10]$, the $\ell_2$ error (expectedly) doubles when we use the~\cite{cai2021cost}
algorithm; in contrast, the performance of~\cref{alg:k2_logR_algo_peeling} is
effectively unchanged.

Finally, we highlight a weakness of~\cref{alg:k2_logR_algo_peeling} -- we gain the 
mild dependence on $R$ at the cost of losing the $\left(\frac{1}{\alpha\epsilon} +
\frac{1}{\alpha^2}\right)$-like sub-Gaussian rate; instead we have
heavy-tailed-style $\left(\frac{1}{\alpha^2\epsilon}\right)$-like rate. This
comes with a practical cost in scenarios where the standard deviation $\sigma$
is much larger than the a priori mean estimate $R$.

For all figures, we average results over $10$ random seeds and report average
results, together with $95\%$ bootstrap confidence intervals. Code necessary
to reproduce all plots is available at
\url{https://anonymous.4open.science/r/sparse_private_mean_estimation-ED75}.
For all experiments we use commodity hardware (CPU: Intel® Core™ i7-9750H  
CPU @ 2.60GHz).

\subsection{Proof of~{{\cref{thm:sparse-alg-peeling}}}}
\label{sec:proof_peeling}
\begin{proof}
    The key idea in our analysis is to show that there exists a parameter regime
    (for the threshold $T$ and bucket size $b$) for which there is a separation
    between the minimum score $z_i$ of coordinates for which $|\mu_i|>T$ and the
    maximum expected score of coordinates for which $\mu_i=0$.

    Assume for simplicity that $n/b=\lceil n/b\rceil$.  Note that
    $m_i\sim\cN(\mu, \sigma^2/bI)$. Define $p_i := \P(|(m_1)_i| > T)$ as the
    probability of a (bucketed mean of a) sample exceeding the given threshold
    $T$.  Then $z_i\sim \text{Bin}(n/b, p_i)$. Note that from Chernoff we have
    for all coordinates $i$ outside the support of $\mu$ that $p_i\leq
    \exp\left(-bT^2/(2\sigma^2)\right)=:p_0$. Let $z_0\sim \text{Bin}(n/b, p_0)$.
    Since all $p_i$s are bounded below by a constant, we can use the normal
    approximation to $z_i$ and sub-Gaussian concentration.  
    On one hand,
    \begin{equation} 
        \begin{split}
        \E \max_{i:\mu_i=0} z_i &\leq \frac{n}{b} p_0 +
                \sqrt{2(\log d + \log(4k/\beta)) \cdot \var(z_0)} \\
        &\leq \frac{n}{b}\exp\left(\frac{-bT^2}{2\sigma^2}\right)
        \left(1 + \sqrt{\frac{2b(\log d + \log(4k/\beta))}{np_0}}\right) 
        \end{split}
    \end{equation}
    with probability at least $1-\beta/(4k)$.

    On the other, for large coordinates $i$ where $\mu_i\geq T$, we have that
    $p_i \geq \frac{1}{2}$ and thus
    \begin{equation}
        \begin{split}
            \E \min_{i:|\mu_i|\geq T} z_i &\geq \frac{n}{b}\min_{i: |\mu_i|\geq T} p_i -
                \sqrt{2(\log d + \log(4k/\beta)) \cdot (n/b)\cdot \min_{i: |\mu_i|\geq T} p_i(1-p_i)} \\
                &\geq \frac{n}{b}\left( \frac{1}{2} - 
                \sqrt{\frac{2b(\log d + \log(4k/\beta))}{n}} \right),
        \end{split}
    \end{equation}
    again with probability at least $1-\beta/(4k)$.

    We can ensure that $\E \min_{i:|\mu_i|\geq T} z_i -
    \frac{2k(\log(d)+\log(4k/\beta))}{\epsilon} > \E \max_{i: \mu_i=0} z_i$, 
    by setting $T = 3.5 {\sigma}/{\sqrt{b}}$ and 
    $n\geq \max\{1, k/\epsilon\} \cdot 20b(\log (d)+\log(4k/\beta))$. For all
    practical purposes, we can safely assume $k/\epsilon \geq 1$ and thus get
    $n\geq 20kb(\log (d)+\log(4k/\beta))/\epsilon$. For all

    In $k$ rounds of the exponential mechanism, each with privacy budget
    $\epsilon/k$, we will pick each coordinate $i$ such that $|\mu_i|\geq T$
    with probability at least $1-\beta/(4k)$. Taking a union bound over the
    success probabilities of the bounds, and the exponential mechanism rounds,
    we get that we will select $k$ coordinates above the threshold $T$ with
    probability $1-\beta/2$ (if there are at least that many). Assuming we
    ``give up'' on potentially non-zero coordinates with mean magnitude lower
    than $T$, we get the following bound on the estimation error 
    $\alpha$.   
    \begin{equation}
        \sum_{i:|\mu_i|\leq T} \mu_i^2 \leq kT^2.
    \end{equation}
    To get to an estimation error of at most $\alpha$, we need
    \begin{equation}
        \begin{split}
        \alpha^2 &\geq kT^2 
        \geq 15 k\frac{\sigma^2}{b} 
        \geq 300k \sigma^2 \cdot \frac{k(\log d+\log(4k/\beta))}{n\epsilon}.
        \end{split}
    \end{equation}
    This implies that for the support estimation part we will need
    \begin{equation}
        n = \Omega\left(\sigma^2\frac{k^2(\log d+\log(k/\beta))}{\alpha^2\epsilon}\right) 
    \end{equation}
    samples. It remains to privately estimate the mean on the selected
    coordinates. Since we have already used up a super-linear (in $k$) number of
    samples, we can afford to use a naive estimator for the mean estimation
    part: we can invoke the histogram-based univariate estimator of Karwa and
    Vadhan~\cite{karwa2017finite} for each coordinate, and require accuracy of
    $\alpha/\sqrt{k}$ on each coordinate. For non-selected coordinates, we
    return $\hat\mu_i=0$. In total, that would give us an $\ell_2$ guarantee of
    $\alpha$, as desired.

    From~\cite[Theorem 1.1]{karwa2017finite} we know that we need 
    \begin{equation}
        n = \Omega\left( \frac{\sigma^2k\log(2k/\beta)}{\alpha^2} +
         \frac{\sigma k^{1.5}\log(2k/\beta)}{\alpha\epsilon} + \frac{k\log(R)}{\epsilon} \right)
    \end{equation}
    samples to estimate each coordinate up to under $\epsilon/k$ DP with
    probability at least $1-\beta/(2k)$. Using 
    coordinate, we will need $\tilde \Omega(k^{1.5}/(\alpha^2\epsilon))$ samples
    in total for the dense mean estimation part.

    Combining the support estimation and the dense mean estimation components,
    we get the desired result.

\end{proof}

%% file: appendicies/restate_cost_of_privacy.tex
\section{Re-statement of sparse mean algorithm of~\cite{cai2021cost}}
\label{sec:restate-cost-of-privacy} 
 
In~\cref{sec:sparse-mean} and~\cref{sec:experimental-details} we
compare~\cref{alg:k2_logR_algo_peeling} with~\cite[Algorithm 3.3]{cai2021cost}
(referred to as CWZ algorithm hereafter).
We focus solely on pure DP, and~\cite{cai2021cost} state only an approximate DP
version of the CWZ algorithm.
To have a fair comparison, we make a (minor) modification of the CWZ algorithm
to handle the $\epsilon$-DP case. In short, we replace the advanced composition
step in their analysis with a basic composition needed for the more stringent
pure DP requirements.

More concretely, we keep~\cite[Algorithm 3.3]{cai2021cost} intact and only
modify the peeling subroutine~\cite[Algorithm 3.2]{cai2021cost}. Given an a
priori boun $R_\infty$ satisfying $\|\mu\|_\infty\leq R_\infty$, we replace the
scale $R_\infty\cdot \frac{\sqrt{k\log(1/\delta)}}{\epsilon}$ of the Laplace
noise added on lines $3$ and $7$ of Algorithm 3.2 with $R_\infty\cdot
\frac{k}{\epsilon}$. The new privacy analysis differs from~\cite[Lemma
3.3]{cai2021cost} only in the composition steps.

%% file: appendicies/concentrated_DP.tex
\section{Meta-theorem for concentrated DP}\label{sec:zCDP}
An alternative formulation of differential privacy if that of zero-concentrated
DP (zCDP) based on R\'enyi divergence. More formally, we have the following
definitions.
\begin{definition}[R\'enyi divergence]
    Given distributions $P$ and $Q$ on a common sample space $\Omega$, the
    $\alpha$-R\'enyi divergence between $P$ and $Q$ is given by
    \begin{equation}
        D_\alpha(P\| Q) = \frac{1}{\alpha - 1}\log\left(\mathop{\E}_{x\sim P}
        \left[\left(\frac{P(x)}{Q(x)}\right)^{\alpha - 1}\right]\right). 
    \end{equation}
\end{definition}

\begin{definition}[zCDP~\cite{bun2016concentrated}]
Let $\cX$ be a set and $\cX^* = \{(X_1,\ldots,X_n) \, : \, n \in \N, X_i \in
\cX\}$ be all possible datasets over $\cX$. For $\rho > 0$, a (randomized)
map $M \, : \, \cX^* \rightarrow \cO$ (where $\cO$ is an \emph{output} set) is
$\rho$-zCDP if for every $(X_1,\ldots,X_n), (X_1',\ldots,X_n') \in \cX^*$
such that $X_i = X_i'$ except for a single index $i$ and for every $\alpha
\in (1, \infty)$ it holds that $D_\alpha(M(X_1,\ldots,
X_n)\|M(X_1',\ldots,X_n')) \leq \rho\alpha$.
\end{definition}

With this, we are ready to state the zCDP version of~\cref{thm:meta}.

\begin{corollary}[zCDP version of~\cref{thm:meta}]
    Let $M \, : \, \cX^* \rightarrow \cO$ be an $\rho$-zCDP map
    from datasets $\cX^*$ to outputs $\cO$.  For every dataset $X_1,\ldots,X_n$,
    let $G_{X_1,\ldots,X_n} \subseteq \cO$ be a set of \emph{good} outputs.
    Suppose that $M(X_1,\ldots,X_n) \in G_{X_1,\ldots,X_n}$ with probability at
    least $1-\beta$ for some $\beta = \beta(n)$.  Then, for every $n \in \N$ and
    every $\delta >0$, on $n$-element datasets $M$ is \emph{robust} to
    adversarial corruption of any $\eta(n)$-fraction of inputs, where
    \[
        \eta(n) = O \Paren{ \min \Paren{ \frac{\log 1/\beta}{\e(\rho, \delta) \cdot n}, \frac{\log 1/\delta}{\e(\rho, \delta) \cdot n + \log n} } }\mcom 
    \]
    and $\e(\rho, \delta) = \rho + 2\sqrt{\rho\log(1/\delta)}$, meaning that for every $X_1,\ldots,X_n$ and $X_1',\ldots,X_n'$ differing on
    only $\eta n$ elements, $M(X_1',\ldots,X_n') \in G_{X_1,\ldots,X_n}$ with
    probability at least $1- \beta^{\Omega(1)}$.
\end{corollary}
\begin{proof}
    The result follows directly from the connection between zCDP and approximate
    differential privacy~\cite[Proposition 1.3]{bun2016concentrated}.
\end{proof}

%% file: sections/acknowledgements.tex
\section{Acknowledgements}

\par Work supported in part by the NSF grants CCF-1553428 and CNS-1815221. This
material is based upon work supported by the Defense Advanced Research Projects
Agency (DARPA) under Contract No. HR001120C0015.

\noindent The authors would like to thank Shyam Narayanan for helpful discussions.

%% file: main.bbl
\newcommand{\etalchar}[1]{$^{#1}$}
\begin{thebibliography}{AACM{\etalchar{+}}22}

\bibitem[AACM{\etalchar{+}}22]{abowd2022geographic}
John~M Abowd, Robert Ashmead, Ryan Cumings-Menon, Daniel Kifer, Philip Leclerc,
  Jeffrey Ocker, Michael Ratcliffe, and Pavel Zhuravlev.
\newblock Geographic spines in the 2020 census disclosure avoidance system
  topdown algorithm.
\newblock {\em arXiv preprint arXiv:2203.16654}, 2022.

\bibitem[ADK{\etalchar{+}}19]{amin2019differentially}
Kareem Amin, Travis Dick, Alex Kulesza, Andres Munoz, and Sergei Vassilvitskii.
\newblock Differentially private covariance estimation.
\newblock {\em Advances in Neural Information Processing Systems}, 32, 2019.

\bibitem[AV19]{ahsen2019approach}
Mehmet~Eren Ahsen and Mathukumalli Vidyasagar.
\newblock An approach to one-bit compressed sensing based on probably
  approximately correct learning theory.
\newblock 2019.

\bibitem[BB20]{brennan2020reducibility}
Matthew Brennan and Guy Bresler.
\newblock Reducibility and statistical-computational gaps from secret leakage.
\newblock In {\em Conference on Learning Theory}, pages 648--847. PMLR, 2020.

\bibitem[BDLS17]{balakrishnan2017computationally}
Sivaraman Balakrishnan, Simon~S Du, Jerry Li, and Aarti Singh.
\newblock Computationally efficient robust sparse estimation in high
  dimensions.
\newblock In {\em Conference on Learning Theory}, pages 169--212. PMLR, 2017.

\bibitem[BFR{\etalchar{+}}00]{batu2000testing}
Tugkan Batu, Lance Fortnow, Ronitt Rubinfeld, Warren~D Smith, and Patrick
  White.
\newblock Testing that distributions are close.
\newblock In {\em Proceedings 41st Annual Symposium on Foundations of Computer
  Science}, pages 259--269. IEEE, 2000.

\bibitem[BGS{\etalchar{+}}21]{brown2021covariance}
Gavin Brown, Marco Gaboardi, Adam Smith, Jonathan Ullman, and Lydia
  Zakynthinou.
\newblock Covariance-aware private mean estimation without private covariance
  estimation.
\newblock {\em Advances in Neural Information Processing Systems}, 34, 2021.

\bibitem[BIPW10]{ba2010lower}
Khanh~Do Ba, Piotr Indyk, Eric Price, and David~P Woodruff.
\newblock Lower bounds for sparse recovery.
\newblock In {\em Proceedings of the twenty-first annual ACM-SIAM symposium on
  Discrete Algorithms}, pages 1190--1197. SIAM, 2010.

\bibitem[BKSW19]{bun2019private}
Mark Bun, Gautam Kamath, Thomas Steinke, and Steven~Z Wu.
\newblock Private hypothesis selection.
\newblock {\em Advances in Neural Information Processing Systems}, 32, 2019.

\bibitem[BS14]{barak2014sum}
Boaz Barak and David Steurer.
\newblock Sum-of-squares proofs and the quest toward optimal algorithms.
\newblock {\em arXiv preprint arXiv:1404.5236}, 2014.

\bibitem[BS16]{bun2016concentrated}
Mark Bun and Thomas Steinke.
\newblock Concentrated differential privacy: Simplifications, extensions, and
  lower bounds.
\newblock In {\em Theory of Cryptography Conference}, pages 635--658. Springer,
  2016.

\bibitem[Bun20]{bun2020computational}
Mark Bun.
\newblock A computational separation between private learning and online
  learning.
\newblock {\em Advances in Neural Information Processing Systems},
  33:20732--20743, 2020.

\bibitem[CFB19]{cherapanamjeri2019fast}
Yeshwanth Cherapanamjeri, Nicolas Flammarion, and Peter~L Bartlett.
\newblock Fast mean estimation with sub-gaussian rates.
\newblock In {\em Conference on Learning Theory}, pages 786--806. PMLR, 2019.

\bibitem[CM13]{cai2013optimal}
T~Tony Cai and Zongming Ma.
\newblock Optimal hypothesis testing for high dimensional covariance matrices.
\newblock {\em Bernoulli}, 19(5B):2359--2388, 2013.

\bibitem[CS22]{chhor2022robust}
Julien Chhor and Flore Sentenac.
\newblock Robust estimation of discrete distributions under local differential
  privacy.
\newblock {\em arXiv preprint arXiv:2202.06825}, 2022.

\bibitem[CWZ21]{cai2021cost}
T~Tony Cai, Yichen Wang, and Linjun Zhang.
\newblock The cost of privacy: Optimal rates of convergence for parameter
  estimation with differential privacy.
\newblock {\em The Annals of Statistics}, 49(5):2825--2850, 2021.

\bibitem[De12]{de2012lower}
Anindya De.
\newblock Lower bounds in differential privacy.
\newblock In {\em Theory of cryptography conference}, pages 321--338. Springer,
  2012.

\bibitem[DK19]{diakonikolas2019recent}
Ilias Diakonikolas and Daniel~M Kane.
\newblock Recent advances in algorithmic high-dimensional robust statistics.
\newblock {\em arXiv preprint arXiv:1911.05911}, 2019.

\bibitem[DK21]{diakonikolas2021sample}
Ilias Diakonikolas and Daniel~M Kane.
\newblock The sample complexity of robust covariance testing.
\newblock In {\em Conference on Learning Theory}, pages 1511--1521. PMLR, 2021.

\bibitem[DKS17]{diakonikolas2017statistical}
Ilias Diakonikolas, Daniel~M Kane, and Alistair Stewart.
\newblock Statistical query lower bounds for robust estimation of
  high-dimensional gaussians and gaussian mixtures.
\newblock In {\em 2017 IEEE 58th Annual Symposium on Foundations of Computer
  Science (FOCS)}, pages 73--84. IEEE, 2017.

\bibitem[DL09]{dwork2009differential}
Cynthia Dwork and Jing Lei.
\newblock Differential privacy and robust statistics.
\newblock In {\em Proceedings of the forty-first annual ACM symposium on Theory
  of computing}, pages 371--380, 2009.

\bibitem[DMNS06]{dwork2006calibrating}
Cynthia Dwork, Frank McSherry, Kobbi Nissim, and Adam Smith.
\newblock Calibrating noise to sensitivity in private data analysis.
\newblock In {\em Theory of cryptography conference}, pages 265--284. Springer,
  2006.

\bibitem[DMR18]{devroye2018total}
Luc Devroye, Abbas Mehrabian, and Tommy Reddad.
\newblock The total variation distance between high-dimensional gaussians.
\newblock {\em arXiv preprint arXiv:1810.08693}, 2018.

\bibitem[DSZ18]{dwork2018differentially}
Cynthia Dwork, Weijie~J Su, and Li~Zhang.
\newblock Differentially private false discovery rate control.
\newblock {\em arXiv preprint arXiv:1807.04209}, 2018.

\bibitem[GHRU13]{gupta2013privately}
Anupam Gupta, Moritz Hardt, Aaron Roth, and Jonathan Ullman.
\newblock Privately releasing conjunctions and the statistical query barrier.
\newblock {\em SIAM Journal on Computing}, 42(4):1494--1520, 2013.

\bibitem[H{\aa}s01]{haastad2001some}
Johan H{\aa}stad.
\newblock Some optimal inapproximability results.
\newblock {\em Journal of the ACM (JACM)}, 48(4):798--859, 2001.

\bibitem[HBK22]{hayes2022learning}
Jamie Hayes, Borja Balle, and M~Pawan Kumar.
\newblock Learning to be adversarially robust and differentially private.
\newblock {\em arXiv preprint arXiv:2201.02265}, 2022.

\bibitem[HKM21]{hopkins2021efficient}
Samuel~B Hopkins, Gautam Kamath, and Mahbod Majid.
\newblock Efficient mean estimation with pure differential privacy via a
  sum-of-squares exponential mechanism.
\newblock {\em arXiv preprint arXiv:2111.12981}, 2021.

\bibitem[HLZ20]{hopkins2020robust}
Sam Hopkins, Jerry Li, and Fred Zhang.
\newblock Robust and heavy-tailed mean estimation made simple, via regret
  minimization.
\newblock {\em Advances in Neural Information Processing Systems},
  33:11902--11912, 2020.

\bibitem[Hop20]{hopkins2020mean}
Samuel~B Hopkins.
\newblock Mean estimation with sub-gaussian rates in polynomial time.
\newblock {\em The Annals of Statistics}, 48(2):1193--1213, 2020.

\bibitem[HSR{\etalchar{+}}08]{homer2008resolving}
Nils Homer, Szabolcs Szelinger, Margot Redman, David Duggan, Waibhav Tembe,
  Jill Muehling, John~V Pearson, Dietrich~A Stephan, Stanley~F Nelson, and
  David~W Craig.
\newblock Resolving individuals contributing trace amounts of dna to highly
  complex mixtures using high-density snp genotyping microarrays.
\newblock {\em PLoS genetics}, 4(8):e1000167, 2008.

\bibitem[HT10]{hardt2010geometry}
Moritz Hardt and Kunal Talwar.
\newblock On the geometry of differential privacy.
\newblock In {\em Proceedings of the forty-second ACM symposium on Theory of
  computing}, pages 705--714, 2010.

\bibitem[Hub65]{huber1965robust}
Peter~J Huber.
\newblock A robust version of the probability ratio test.
\newblock {\em The Annals of Mathematical Statistics}, pages 1753--1758, 1965.

\bibitem[Hub11]{huber2011robust}
Peter~J Huber.
\newblock Robust statistics.
\newblock In {\em International encyclopedia of statistical science}, pages
  1248--1251. Springer, 2011.

\bibitem[KLN{\etalchar{+}}11]{kasiviswanathan2011can}
Shiva~Prasad Kasiviswanathan, Homin~K Lee, Kobbi Nissim, Sofya Raskhodnikova,
  and Adam Smith.
\newblock What can we learn privately?
\newblock {\em SIAM Journal on Computing}, 40(3):793--826, 2011.

\bibitem[KMV21]{kothari2021private}
Pravesh~K Kothari, Pasin Manurangsi, and Ameya Velingker.
\newblock Private robust estimation by stabilizing convex relaxations.
\newblock {\em arXiv preprint arXiv:2112.03548}, 2021.

\bibitem[KSS94]{kearns1994toward}
Michael~J Kearns, Robert~E Schapire, and Linda~M Sellie.
\newblock Toward efficient agnostic learning.
\newblock {\em Machine Learning}, 17(2):115--141, 1994.

\bibitem[KSU20]{kamath2020private}
Gautam Kamath, Vikrant Singhal, and Jonathan Ullman.
\newblock Private mean estimation of heavy-tailed distributions.
\newblock In {\em Conference on Learning Theory}, pages 2204--2235. PMLR, 2020.

\bibitem[KV17]{karwa2017finite}
Vishesh Karwa and Salil Vadhan.
\newblock Finite sample differentially private confidence intervals.
\newblock {\em arXiv preprint arXiv:1711.03908}, 2017.

\bibitem[LKKO21]{liu2021robust}
Xiyang Liu, Weihao Kong, Sham Kakade, and Sewoong Oh.
\newblock Robust and differentially private mean estimation.
\newblock {\em Advances in Neural Information Processing Systems}, 34, 2021.

\bibitem[LKO21]{liu2021differential}
Xiyang Liu, Weihao Kong, and Sewoong Oh.
\newblock Differential privacy and robust statistics in high dimensions.
\newblock {\em arXiv preprint arXiv:2111.06578}, 2021.

\bibitem[LM19]{lugosi2019sub}
G{\'a}bor Lugosi and Shahar Mendelson.
\newblock Sub-gaussian estimators of the mean of a random vector.
\newblock {\em The annals of statistics}, 47(2):783--794, 2019.

\bibitem[LT91]{ledoux1991probability}
Michel Ledoux and Michel Talagrand.
\newblock {\em Probability in Banach Spaces: isoperimetry and processes},
  volume~23.
\newblock Springer Science \& Business Media, 1991.

\bibitem[MT07]{mcsherry2007mechanism}
Frank McSherry and Kunal Talwar.
\newblock Mechanism design via differential privacy.
\newblock In {\em 48th Annual IEEE Symposium on Foundations of Computer Science
  (FOCS'07)}, pages 94--103. IEEE, 2007.

\bibitem[MV21]{mangoubi2021sampling}
Oren Mangoubi and Nisheeth~K Vishnoi.
\newblock Sampling from log-concave distributions with infinity-distance
  guarantees and applications to differentially private optimization.
\newblock {\em arXiv preprint arXiv:2111.04089}, 2021.

\bibitem[NRS07]{nissim2007smooth}
Kobbi Nissim, Sofya Raskhodnikova, and Adam Smith.
\newblock Smooth sensitivity and sampling in private data analysis.
\newblock In {\em Proceedings of the thirty-ninth annual ACM symposium on
  Theory of computing}, pages 75--84, 2007.

\bibitem[RSS18]{raghavendra2018high}
Prasad Raghavendra, Tselil Schramm, and David Steurer.
\newblock High dimensional estimation via sum-of-squares proofs.
\newblock In {\em Proceedings of the International Congress of Mathematicians:
  Rio de Janeiro 2018}, pages 3389--3423. World Scientific, 2018.

\bibitem[SU15]{steinke2015between}
Thomas Steinke and Jonathan Ullman.
\newblock Between pure and approximate differential privacy.
\newblock {\em arXiv preprint arXiv:1501.06095}, 2015.

\bibitem[TS13]{thakurta2013differentially}
Abhradeep~Guha Thakurta and Adam Smith.
\newblock Differentially private feature selection via stability arguments, and
  the robustness of the lasso.
\newblock In {\em Conference on Learning Theory}, pages 819--850. PMLR, 2013.

\bibitem[Tuk75]{tukey1975mathematics}
John~W Tukey.
\newblock Mathematics and the picturing of data.
\newblock In {\em Proceedings of the International Congress of Mathematicians,
  Vancouver, 1975}, volume~2, pages 523--531, 1975.

\bibitem[UKRK22]{usynin2022collaborative}
Dmitrii Usynin, Helena Klause, Daniel Rueckert, and Georgios Kaissis.
\newblock Can collaborative learning be private, robust and scalable?, 2022.

\bibitem[Ull16]{ullman2016answering}
Jonathan Ullman.
\newblock Answering n\^{}2+o(1) counting queries with differential privacy is
  hard.
\newblock {\em SIAM Journal on Computing}, 45(2):473--496, 2016.

\bibitem[UV11]{ullman2011pcps}
Jonathan Ullman and Salil Vadhan.
\newblock Pcps and the hardness of generating private synthetic data.
\newblock In {\em Theory of Cryptography Conference}, pages 400--416. Springer,
  2011.

\end{thebibliography}
